\documentclass[runningheads]{llncs}
\usepackage[T1]{fontenc}

\usepackage{algorithm}
\usepackage[noend]{algpseudocode}
\usepackage[inline]{enumitem}
\usepackage{mathtools}
\usepackage{dsfont}
\usepackage{todonotes}

\usepackage{amsmath}
\usepackage{amssymb}
\usepackage{tikz}
\usetikzlibrary{positioning}

\newfloat{listing}{htbp}{lop}
\floatname{listing}{Listing}

\usepackage{graphicx}

\usepackage{hyperref}
\usepackage[capitalise]{cleveref}
\usepackage{booktabs}
\usepackage{listings}
\usepackage{float}

\AddToHook{cmd/appendix/before}{\crefalias{section}{appendix}}
\AddToHook{cmd/appendix/before}{\crefalias{subsection}{appendix}}

\crefname{definition}{Def.}{Defs.}
\Crefname{definition}{Definition}{Definitions}
\crefname{appendix}{App.}{Apps.}
\Crefname{appendix}{Appendix}{Appendices}
\crefname{theorem}{Thm.}{Thms.}
\Crefname{theorem}{Theorem}{Theorems}
\crefname{algorithm}{Alg.}{Algs.}
\Crefname{algorithm}{Algorithm}{Algorithms}

\newcommand{\ran}{\ensuremath{\mathrm{ran}}}

\begin{document}

\title{The Curious Case of Exploding DecPOMDPs:\\ Containing the Fire through Policy Counting\thanks{Full version including appendix of a paper accepted at the 17th International Conference on Scalable Uncertainty Management (SUM2026) under the same name.}}
\titlerunning{The Curious Case of Exploding DecPOMDPs}
\author{Nazl\i\ Nur Karabulut\orcidID{0000-0001-7958-5627} \and
Tanya Braun\orcidID{0000-0003-0282-4284} }
\authorrunning{N.\ N.\ Karabulut and T. Braun}
\institute{Computer Science Department, University of Münster, Münster, Germany 
\email{\{nnur.karabulut,tanya.braun\}@uni-muenster.de}
}
\maketitle
\begin{abstract}
Decentralised partially observable Markov decision processes (DecPOMDPs) provide a general framework for modelling multi-agent decision making under uncertainty.
However, DecPOMDPs are known to suffer from exponential complexity in the number of agents.
One way to combat this intractability in agent numbers is to look at partitions of agents that exhibit a form of symmetry among agents, allowing for a compact encoding by counting.
However, a challenge arises as the policy space explodes, even though the model complexity and evaluation cost reduce to a polynomial dependence. 
In this paper, we redirect our focus from counting agents to counting policies, which actually enables tractability in agent numbers for so called policy-counted DecPOMDPs.
Further, we present policy-counted dynamic programming using the compact representation to solve policy-counted DecPOMDPs efficiently.
\end{abstract}
\section{Introduction}
Large-scale agent systems involve a huge number of interacting agents that have to coordinate decisions in uncertain environments. 
To formalise the setting, decentralised partially observable Markov decision processes (DecPOMDPs) provide a framework to compute a joint policy for a set of cooperative agents that maximises some joint utility function in an environment that is considered to be stationary and describable by a probabilistic transition model. 
When agents are self-interested, the problem is formalised as a partially observable stochastic game (POSG), in which each agent has its own utility function.

However, solving multi-agent decision making problems is notoriously difficult as the number of potential policies escalates exponentially with the number of agents.
To deal with this problem, researchers have proposed various approaches:
Early work focuses on dynamic programming using pruning techniques for POSGs \cite{HanBeZu04}.
Building on this basis, Szer et al.\ present point-based dynamic programming for DecPOMDPs \cite{szer2006point} and Seuken et al.\ propose memory-bounded dynamic programming \cite{seuken2007memory}.
The dynamic programming operator was later extended with approximate pruning techniques to improve scalability for POSGs \cite{KumZi09}.
In addition to dynamic programming–based approaches, Szer et al.\ \cite{szer2012maa} propose multiagent A* (MAA*)\pagebreak, a policy search algorithm that groups action–observation histories of agents at the same stage when they share the same optimal Q-values in a Bayesian game formulation of the problem \cite{oliehoek2009lossless}.
Newer work optimises MAA* even further \cite{koops2023recursive,koops2024approximate}.
Additional research has examined more structured subclasses to improve scalability, e.g., for POSGs, examples include zero-sum formulations \cite{WigOlRo16,HorBo19,TomHoArBoCh21}, one-sided POSGs where only one agent is affected by uncertainty \cite{HorBoPe17,HorBoKiKa19,CarJaBhSpTo21}, and common-payoff POSGs that restrict agents to share identical rewards \cite{EmeGoScTh04}.

Another lane of work focuses on so-called lifting \cite{Poo03}.
Lifting exploits symmetries between objects, also sometimes referred to as indistinguishability or exchangeability, which allows for tractable inference \cite{NieBr14} and has been used to great extent in probabilistic inference (e.g., \cite{MilZeKeHaKa08,BroTaMeDaRa11,TagFiDaBl13,BraMo16a}) and single-agent decision making (e.g., \cite{SanBo09,SanKe10,MarBrMoGe26}).
In multi-agent systems, lifting is applied to agents, allowing for a more compact problem formulation using partitions of agents.
The basic assumption, first described by Braun et al.\ \cite{BraGeLaMo22}, is that, within a partition, permutations of joint actions and observations have the same effect on the joint transition or reward, allowing for counting actions and observations in histograms.
Adding in an assumption about conditional independence among agents of a partition even allows for working with a single representative agent per partition without counting necessary, leading to a drastic reduction in complexity for so-called isomorphic DecPOMDPs. 
However, without that assumption, a reduction can only be made in model complexity and evaluation cost, while ---almost counter-intuitively--- the policy space explodes \cite{BraGeLaMo22}. 
More recently, we have presented isomorphic POSGs and lifted dynamic programming to work with the compact encoding, yielding a runtime polynomial in the number of agents \cite{karabulut2024lifting}.
We have extended this work to counting policies to circumvent an explosion in the policy space for partitioned POSGs \cite{karabulut2025counting}.

Invigorated by the progress made in counting POSGs, this work takes on the case of exploding DecPOMDPs under counting, a challenge left open by Braun et al.\ \cite{BraGeLaMo22} by moving the counting focus to policies. 
Specifically, the contributions are threefold:
\begin{enumerate*}[label=(\roman*)]
    \item a new definition of counting DecPOMDPs that allows for counting policies, preventing the explosion in the policy space,
    \item an updated utility calculation including an analysis that shows tractability in agent numbers, and
    \item a dynamic programming operator to solve policy-counted DecPOMDPs.
\end{enumerate*}
As such, we solve the above-mentioned challenge and additionally provide a dedicated solution method.
When showing tractability, we assume a fixed number of partitions that is much smaller than the number of agents.

The paper is organised as follows: 
We start with DecPOMDPs and dynamic programming.
Then we analyse the problem of exploding DecPOMDPs and show how to use counting for policies to circumvent the problem, followed by counting dynamic programming.
We end with a conclusion. 

\section{Preliminaries}
In this section, we define DecPOMDPs and provide an overview of dynamic programming for DecPOMDPs. 

\subsection{Decentralised POMDPs}
The definition is based on \cite{OliAm16} and \cite{RusNo21}, using random variables $S$ with $\ran(S)$ referring to the set of (range) values it can take. 
Bold symbols denote sets, (time) steps $t$ are denoted by superscript $t$, sequences over a discrete interval $[t_s:t_e]$ are denoted by superscript ${t_s:t_e}$, and omitting an element $i$ from a set or sequence is denoted by subscript $-i$.
To add an element to a sequence, we use $\circ$.
\begin{definition}\label{def:decpomdp:gr}
	A (\emph{ground}) \emph{DecPOMDP} $M$ is a tuple $(\boldsymbol{I},\allowbreak S,\allowbreak  \boldsymbol{A},\allowbreak T,\allowbreak R, \allowbreak \boldsymbol{O}, \varOmega, \tau )$, 
	\begin{itemize}
	   \item $\boldsymbol{I}$ a set of $N$ agents, 
		\item $S$ a random variable with a set of states as a range,
        \item $\boldsymbol{A} = \{A_i\}_{i\in \boldsymbol{I}}$ a set of decision random variables, each $A_i$ with a set of local actions as range,
        \item $T(S', S, \boldsymbol{A}) = P(S' \mid S, \boldsymbol{A})$ a transition function,
        \item $R(S, \boldsymbol{A})$ a reward function,
        \item $\boldsymbol{O} = \{O_i\}_{i\in \boldsymbol{I}}$ a set of random variables, each $O_i$ with a set of local observations as range,  
        \item $\varOmega(\boldsymbol{O}, S) = P(\boldsymbol{O} \mid S)$ a sensor function\footnote{For ease of exposition regarding counting and to be consistent with \cite{BraGeLaMo22}, we omit $\boldsymbol{A}$ from $\varOmega$, which does not make the problem inherently easier to solve \cite{RusNo21}.}, and
        \item $\tau$ a horizon. 
	\end{itemize}
\end{definition}
Each agent $i$ follows a local policy $\pi_i^t$ mapping local observation histories $o^{0:t}$ to actions $a$, which can be represented as a depth-$t$ policy tree (see \cref{fig:policies}). 
A joint policy is a tuple $ \boldsymbol{\pi}=(\pi_1, ..., \pi_n)$.
The semantics of the ground model $M$ is given by the joint policy space $\varPi_M$.
The DecPOMDP problem asks for the joint policy that yields the maximum expected utility (with horizon $\tau$):
\begin{align}
    MEU (M) = (\operatorname*{arg\;max}_{\boldsymbol{\pi} \in \varPi_M}  U_{M}(\boldsymbol{\pi}) , \max_{\boldsymbol{\pi} \in \varPi_M}  U_{M}(\boldsymbol{\pi}))  \label{meuforgroundcase}
\end{align}
where $U_{M}(\boldsymbol{\pi})$ is calculated recursively over $\tau$ steps, by summing over joint observations and next states, and then weighted according to the prior $T(s^0, .\ , .\ )$, i.e., $U_M(\boldsymbol{\pi}) = \sum_{s^0 \in \ran(S)} T(s^0, .\  ,.\ )U_M ^{\boldsymbol{\pi}} (s^0, \bot^{0:0})$ with
\begin{align}\label{eq:eu:gr}
    \begin{split}
        U_M ^{\boldsymbol{\pi}}(s^t, \boldsymbol{o}^{0:t}) = R(s_t, \boldsymbol{\pi}(\boldsymbol{o}^{0:t})) +& \sum_{s^{t+1} \in \ran(S)} T(s^{t+1}, s^t, \boldsymbol{\pi}(\boldsymbol{o}^{0:t})) \\
        &\ \ \ \sum_{\boldsymbol{o}^{t+1} \in \ran(\boldsymbol{O})}
        \varOmega(\boldsymbol{o}^{t+1},s^{t+1}) U_M ^{\boldsymbol{\pi}} (s^{t+1}, \boldsymbol{o}^{0:t+1})
    \end{split}
\end{align}
ending with 
$U_M ^{\boldsymbol{\pi}}(s, \boldsymbol{o}^{0:\tau-1}) = R(s, \boldsymbol{\pi}(o^{0:\tau-1}))$, where $\boldsymbol{\pi}(\boldsymbol{o}^{0:t})$ denotes the joint action at step $t$ ($\boldsymbol{\pi}(\bot)$: root actions) and $\boldsymbol{o}^{0:t+1} = \boldsymbol{o}^{0:t} \circ \boldsymbol{o}^{t+1}$ the updated history.

In a DecPOMDP, all agents share a single joint reward function $R(S, \boldsymbol{A})$ and cooperate to maximize the total expected utility.
However, each agent decides what to do next on its own, not being able to observe the full state or the other agents' observations. 
Nonetheless, agents are not fully independent since they depend on a joint state and receive the reward jointly.

\begin{figure}[tb]
    \centering
    \begin{tikzpicture}[act/.style={draw,circle, minimum width=6mm, node distance=9mm}]
        \node[act] (a0) {a};
        \node[act, below left of=a0] (a01) {a};
        \node[act, below right of=a0] (a02) {a};
        \draw[->] (a0) -- node[above left] {x} (a01);
        \draw[->] (a0) -- node[above right] {y} (a02);

        \node[act, right of=a0, xshift=12mm] (a1) {a};
        \node[act, below left of=a1] (a11) {a};
        \node[act, below right of=a1] (a12) {b};
        \draw[->] (a1) -- node[above left] {x} (a11);
        \draw[->] (a1) -- node[above right] {y} (a12);

        \node[act, right of=a1, xshift=12mm] (a2) {a};
        \node[act, below left of=a2] (a21) {b};
        \node[act, below right of=a2] (a22) {a};
        \draw[->] (a2) -- node[above left] {x} (a21);
        \draw[->] (a2) -- node[above right] {y} (a22);

        \node[act, right of=a2, xshift=12mm] (a3) {a};
        \node[act, below left of=a3] (a31) {b};
        \node[act, below right of=a3] (a32) {b};
        \draw[->] (a3) -- node[above left] {x} (a31);
        \draw[->] (a3) -- node[above right] {y} (a32);

        \node[act, below of=a0, yshift=-5mm] (a4) {b};
        \node[act, below left of=a4] (a41) {a};
        \node[act, below right of=a4] (a42) {a};
        \draw[->] (a4) -- node[above left] {x} (a41);
        \draw[->] (a4) -- node[above right] {y} (a42);

        \node[act, below of=a1, yshift=-5mm] (a5) {b};
        \node[act, below left of=a5] (a51) {a};
        \node[act, below right of=a5] (a52) {b};
        \draw[->] (a5) -- node[above left] {x} (a51);
        \draw[->] (a5) -- node[above right] {y} (a52);

        \node[act, below of=a2, yshift=-5mm] (a6) {b};
        \node[act, below left of=a6] (a61) {b};
        \node[act, below right of=a6] (a62) {a};
        \draw[->] (a6) -- node[above left] {x} (a61);
        \draw[->] (a6) -- node[above right] {y} (a62);

        \node[act, below of=a3, yshift=-5mm] (a7) {b};
        \node[act, below left of=a7] (a71) {b};
        \node[act, below right of=a7] (a72) {b};
        \draw[->] (a7) -- node[above left] {x} (a71);
        \draw[->] (a7) -- node[above right] {y} (a72);

    \end{tikzpicture}
    \caption{Policies at level $t=1$ available to agents with actions $a,b$ and observations $x,y$}
    \label{fig:policies}
\end{figure}
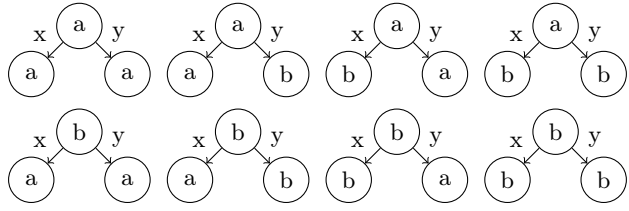

\subsection{Dynamic Programming}
Dynamic programming is a solution method for multi-agent decision making that is shared between DecPOMDPs and POSGs \cite{OliAm16}.
It builds increasingly deeper policy trees from sub-trees for individual agents, while eliminating weakly-dom\-i\-nat\-ed policies at each depth \cite{szer2006point}.
\Cref{alg:dp} shows the dynamic programming operator which consists of three steps.
First, an exhaustive backup is performed to build all possible policies $\varPi_i^{t} = \{\pi_{i,j}^{t}\}_{j=1}^{m_t}$ of depth $t$ from the input policies $\varPi_i^{t-1} = \{\pi_{i,j}^{t-1}\}_{j=1}^{m_{t-1}}$ for each agent $i$.
Then, for each policy $\pi_{i,j}^t \in \varPi_i^t$, the operator calculates the corresponding value vector $V_{i,j}^{t}$, representing the value of that policy for each possible combination of state $s \in \ran(S)$ and policies of the other agents $\boldsymbol{\pi}_{-i}^t \in \boldsymbol{\varPi}_{-i}^{t}$, which is defined as follows with $\boldsymbol{\pi} = \pi_{i,j}^t \circ \boldsymbol{\pi}_{-i}$
\begin{align}
\begin{split}
	V_i(s,\boldsymbol{\pi}) = R(s, \boldsymbol{\pi}(\bot)) +
    {\sum_{s' \in \ran(S)}} T(s', s,\boldsymbol{\pi}(\bot)) 
    {\sum_{\boldsymbol{o} \in \ran(\boldsymbol{O})}} \varOmega(\boldsymbol{o}, s)  V_i(s',\boldsymbol{\pi}.\boldsymbol{o}) \label{eq:valuef}
\end{split}
\end{align}
where $\boldsymbol{\pi}.\boldsymbol{o}$ denotes the sub-policies after following an observation.
During pruning, the operator eliminates policies that are weakly-dominated by others (solvable by a linear programme, see \cref{app:lp:gr}), meaning that over the complete space of $\ran(S) \times \boldsymbol{\varPi}_{-i}^{t}$ there is always another policy with at least as high a value, i.e., \begin{align}
    \begin{split}
	\forall s\in \ran(S), \boldsymbol{\pi}_{-i}^{t} \in \boldsymbol{\varPi}_{-i}^t \exists {\pi}_{i,j'}^{t} \in \varPi_i^{t} : 
    V_i(s, {\pi}_{i,j}^{t} \circ \boldsymbol{\pi}_{-i}^{t}) \leq V_i(s,{\pi}_{i,j'}^{t} \circ \boldsymbol{\pi}_{-i}^{t}) \label{eq:prune}
    \end{split}
\end{align}

\begin{algorithm}[tb]
    \caption{Multi-agent Dynamic Programming Operator}
    \label{alg:dp}
    \begin{algorithmic}
        \Function{MA-DP}{set of policies $\varPi_i^{t-1}$ for each agent $i\in \boldsymbol{I}$ with value vectors $\boldsymbol{V}_i^{t-1}$}
            \For {each agent $i\in \boldsymbol{I}$}
                \State $\varPi_i^{t} \gets $ Perform exhaustive backup using $\varPi_i^{t-1}$ 
                \State $\boldsymbol{V}_i^{t} \gets $ Calculate new value vectors 
            \EndFor
			\While {$\exists \pi_{i,j}^{t} \in \varPi_i^{t} : $ \cref{eq:prune} holds for some $i \in \boldsymbol{I}$}
				\State $\varPi_i^{t} \gets \varPi_i^{t} \setminus \{\pi_{i,j}^{t}\}$, $\boldsymbol{V}_i^{t} \gets \boldsymbol{V}_i^{t} \setminus \{v_{i,j}^t\}$
			\EndWhile
			\State \Return $\{(\varPi_i^t,\boldsymbol{V}_i^t)\}_{i\in \boldsymbol{I}}$
		\EndFunction
	\end{algorithmic}
\end{algorithm}

\section{A Case of Exploding DecPOMDPs}
\label{sec:explode}
The basic idea of lifting DecPOMDPs is that the agent set is partitioned into sets of agents that behave indistinguishably among each other, meaning that it does not matter which agent performs which action within a partition, only how many do so \cite{Poo03,MilZeKeHaKa08,BraGeLaMo22}.
In such a case, it is enough to count the agents and store the counts in a histogram.
To illustrate the point, consider $N=12$ indistinguishable agents with the actions $a, b$ available as an example.
Imagine that $8$ agents perform $a$ while the remaining $4$ perform $b$.
Then, there exist $ \binom{12}{8} = \frac{12!}{8! \cdot 4!} = 495$ ways to make the $12$ agents perform the actions, all of which result in the same outcome $\rho$. 
Therefore, we can construct a histogram $[4,8]$ that maps to $\rho$, replacing $495$ mappings.
While this leads to a drastic decrease in model complexity, the policy space explodes in the formalisation chosen in \cite{BraGeLaMo22}.
To understand why, we next discuss the assumptions and design choices, left at times implicit in the original paper.
Then, we look at the (in)tractability results.

\subsection{Assumptions \& Definitions}
Braun et al.\ \cite{BraGeLaMo22} name as their basic assumption indistinguishability among agents.
They posit that indistinguishability requires agents to have the same local actions and observations, and that another consequence is that permutations of actions (observations) map to the same outcome in transition, reward, and sensor function, which we posit boils down to two assumptions, which we formalise next.

The first assumption is that all agents $i,j$ in a partition $\mathfrak{I} \subseteq \boldsymbol{I}$ have the same action and observation space:
\begin{equation}
	\ran(A_i) = \ran(A_j) \wedge \ran(O_i) = \ran(O_j) \label{eq:assumpt:1}
\end{equation}
which follows the requirement definition in \cite{BraGeLaMo22}.
The second assumption is that the transition, reward, and sensor functions $T$, $R$, and $\varOmega$ are symmetric, meaning that exchanging actions or observations between agents of a partition has no effect on the outcome of $T$, $R$, and $\varOmega$, which can be formally expressed for the actions of two agents as:
\begin{align}
	\shortintertext{$\forall \boldsymbol{a} \in \ran(\boldsymbol{A}), \boldsymbol{a} = (a_i \circ a_j \circ \boldsymbol{a}_{-i,-j}) : $} 
	\begin{split}
		T(s', s, {a}_i \circ {a}_j \circ \boldsymbol{a}_{-i,-j}) &= T(s', s, a_j \circ a_i \circ \boldsymbol{a}_{-i,-j})
\label{eq:assumpt:2}
	\end{split}
\end{align}
The same holds for the actions in $R$ and the observations 
in $\varOmega$.
We now generalise \cref{eq:assumpt:2} to entire partitions, to build interchangeability among agents within a partition.
The set of agents $\boldsymbol{I}$ is divided into $K$ partitions $\mathfrak{I}_1, ...,\mathfrak{I}_K $ of indistinguishable agents.
For each partition $\mathfrak{I}_k$, let $\boldsymbol{a}_k = (a_i)_{i \in \mathfrak{I}_k}$ and $\boldsymbol{o}_k = (o_i)_{i \in \mathfrak{I}_k}$ denote the joint action and observation tuples of the agents in that partition. 
With the extension of symmetric behaviour to partitions, swapping the actions or observations of agents within the same partition does not change a function's outcome.
Formally, for any permutation $\theta$ of $\mathfrak{I}_k$,
\begin{align}
    \begin{split}
    T(s', s, \boldsymbol{a}_k, \boldsymbol{a}_{-k}) &= T(s', s, \theta(\boldsymbol{a}_k),\boldsymbol{a}_{-k})  \boldsymbol{o}_{-k},s).  
    \label{eq:asst2:counting}
    \end{split}
\end{align}
with the same holding for any permutations of actions in $T$ and observations in $\varOmega$.
These assumptions are analogous to those made for counting in POSGs \cite{karabulut2025counting}.

\Cref{eq:assumpt:1,eq:asst2:counting} allow for encoding actions and observations using histograms, which can then be used to compactly encode $T$, $R$, and $\varOmega$ by replacing the individual mappings from all $\theta(\boldsymbol{a}_k)$ to an outcome $\rho$ with a single mapping $h \mapsto \rho$, with $h$ being a histogram of the counts of actions in $\boldsymbol{a}_k$, which is identical for all $\theta(\boldsymbol{a}_k)$.
The same holds for the permutations of observations in $\boldsymbol{o}_k$.
Formally, histograms are defined as follows: 
\begin{definition}
    Given a set of $n$ agents $\boldsymbol{I}$ and a set of $m$ values $\nu = \{v_1, \dots, v_m\}$, a histogram $h$ is defined as $h=\{(v_l,n_l)\}_{l=1}^m$ with $\sum_{l=1} ^{m} n_l = n $ where each $n_l$ denotes the number of instances with value $v_l$.
	We use $h(v_l)$ to refer to $n_l$.
$[n_1, ... ,n_m]$ is used as a shorthand.
	A counting random variable (CRV) $\#_{\boldsymbol{I}} [V]$ is a syntactic construct whose range consists of all histograms over the range values of variable $V$, i.e., $\nu=\ran(V)$, that meet $\sum_l n_l = |\boldsymbol{I}|$. 
\end{definition}
Considering the shorthand example $h=[4,8]$ from above for actions $a,b$ and $n=12$, the full version reads $h=\{(a,8),(b,4)\}$.

Braun et al.\ \cite{BraGeLaMo22} then define a partitioned DecPOMDP that uses such histograms as range values for actions and observations in each partition for every affected component of the DecPOMDP.

\begin{definition}[\cite{BraGeLaMo22}]\label{def:decpomdp:counold}
	A \emph{counting DecPOMDP} $\bar{M}_c$ is a tuple $(\bar{\boldsymbol{I}},\allowbreak \bar{S},\allowbreak  \bar{\boldsymbol{A}}_c,\allowbreak \bar{{T}},\allowbreak \bar{{R}},\allowbreak \bar{\boldsymbol{O}}_c, \allowbreak \bar{{\varOmega}}, \tau)$, with 
	\begin{itemize}
		\item $\bar{\boldsymbol{I}}$ a \emph{partitioning} $\{\mathfrak{I}_k\}_{k=1}^K$ of $N$ agents, with $n_k = |\mathfrak{I}_k|$ and $|\bar{\boldsymbol{I}}| = \sum_{k} n_k = N$,
		\item $S$ a random variable with a set of states as range,
		\item $\bar{\boldsymbol{A}}_c = \{\#_{\mathfrak{I}_k} [{{A}_k}]\}_{k=1}^K$ a set of decision CRVs, 		\item $\bar{T}(S',S,\bar{\boldsymbol{A}}_c)$ a transition function with counted actions, 
		\item $\bar{R}(S, \bar{\boldsymbol{A}}_c)$ a reward function with counted actions, 
		\item $\bar{\boldsymbol{O}}_c = \{\#_{\mathfrak{I}_k} [{{O}_k}]\}_{k=1}^K$ a set of CRVs,
        \item $\bar{\varOmega}(\bar{\boldsymbol{O}}_c, S)$ a sensor function with counted observations, and
        \item $\tau$ a horizon.
	\end{itemize}
\end{definition}
The equivalence between a DecPOMDP fulfilling \cref{eq:assumpt:1,eq:asst2:counting} and a counting Dec\-POMDP can be shown by construction, replacing all mappings with permutations of inputs with a single mapping with a histogram and vice versa (see App.~\ref{app:countingequiv}).

\subsection{(In)Tractability Results}
For a ground DecPOMDP, the model complexity lies in $O(s^2 a^N)$ for $T$, $O(s a^N)$ for $R$, and $O(s o^N)$ for $\varOmega$ with $s=|\ran(S)|$, $a = \max_{i \in \boldsymbol{I}} |\ran(A_i)|$, and $o = \max_{i \in \boldsymbol{I}} |\ran(O_i)|$, while the cost of evaluating a joint policy lies in $O(so^{N\tau})$ and the policy space size in $O(a^{\frac{N(o^\tau-1)}{o-1}})$ \cite{OliAm16}. 
Using CRVs allows for reducing the model size and evaluation cost from exponential to polynomial.
\begin{theorem}[\cite{BraGeLaMo22}, Thm.\ 2]
    A counting DecPOMDP $\bar{M}_c$ allows for representation [i.e., model] and cost to depend polynomially on $N$.
\end{theorem}
The theorem holds by the range sizes of the CRVs, which are capped by $N^a$, $a = \max_{k}|\ran(A_k)|$, and $N^o$, $o = \max_{k}|\ran(O_k)|$, as an upper bound on the number of histograms given by the binomial coefficient $\binom{N+v-1}{v-1} \leq N^v$, $v\in \{a,o\}$ \cite{MilZeKeHaKa08}.
Compared to the ground case, $a$ and $o$ are replaced by $N^a$ and $N^o$, while $K$ replaces $N$ in the exponent, yielding a polynomial dependence on $N$, assuming that $K \ll N$.

However, the number of policies explodes: 
Since the size of the policy space contains $o$ in the exponent, replacing $o$ with $N^o$ means that $N$ remains in the exponent, all the while $a$ is replaced with $N^a$, leading to an expression that has $N$ in the base as well as the exponent. 
\begin{corollary}
    The policy space of a counting DecPOMDP $\bar{M}_c$ depends exponentially on $N$.
\end{corollary}
The crux is the choice of CRVs in the action and observation set used for building policies.
In contrast, we use plain variables and count policies instead, which is a small change with profound effects.

\section{Containing the Fire by Counting Policies}
For a polynomial dependence on the agent numbers, we keep plain variables in the action and observation sets, enabled by \cref{eq:assumpt:1}, but use CRVs in the transition, reward, and sensor function, enabled by \cref{eq:asst2:counting}.
With plain variables, there is a set of (representative) policies available in each partition, which allows for structuring the policy space by counting the number of agents following each representative policy and updating the $MEU$ calculations. 
We first formalise counted policies before renewing the definition of counting DecPOMDPs.
Last, we look at the effect of counted policies on the $MEU$ calculation and complexity.

\subsection{Counted Policies}
Let us consider the example from the beginning of \cref{sec:explode} with $12$ agents and actions $a,b$ available.
At step $t=1$, assuming two possible observations $x,y$, each of the $12$ agents has the same $8$ policies available, see \cref{fig:policies}.
Instead of tracking all possible policies for each individual agent, we can use these $8$ policies as representative policies to keep track of agent behaviour in histograms: 
For example, the partition policy in which all agents follow the first policy is encoded by $[12,0,0,0,0,0,0,0]$, while the other configuration in which one agent follows the second policy and the remaining agents follow the first one is encoded by $[11,1,0,0,0,0,0,0]$. 
In this way, all possible partition policies can be compactly encoded. 

More formally, according to \cref{eq:assumpt:1}, the set of policies for any two agents $i,j$ within a partition $\mathfrak{I}_k$ is established using the same sets of actions $A_i = A_j = A_k $ and observations $O_i = O_j = O_k$. 
As a result, their local policy space is the same, i.e., $\varPi_i= \varPi_j = \varPi_k$.
These policies are referred to as \emph{representative policies} $\varPi_k$ of partition $\mathfrak{I}_k$ and can be counted for each partition.

\begin{definition}\label{def:cntpol}
	Let the set of $J$ {representative policies} of a partition $\mathfrak{I}_k$ be given by $\varPi_k = \{\pi_1, \dots, \pi_J\}$. 
	A \emph{counted partition policy} over $\varPi_k$ is a histogram $h^\varPi_k = \{(\pi_j, m_j)\}_{j=1}^J$, where $m_j \geq 0$ is the number of agents in $\mathfrak{I}_k$ following policy $\pi_j$, and $\sum_{j=1} ^J m_j = n_k$.
	The counted policy space for partition $\mathfrak{I}_k$, denoted by $H^\varPi_k$, is given by the set of all histograms meeting $\sum_{j=1} ^J m_j = n_k$.
	The \emph{joint counted policy} space over all partitions is given by $\boldsymbol{H}^\varPi = \times_{k=1}^K H^\varPi_k$.
	The counted policy space of all partitions except $k$ is denoted by $\boldsymbol{H}^\varPi_{-k}$ and of all agents except some agent $i \in \mathfrak{I}_k$ by $\boldsymbol{H}^\varPi_{-i}$ where $\sum_{j=1} ^J m_j = n_k - 1$ for $H^\varPi_{k}$ without $i$.
\end{definition}

\subsection{A Renewed Definition} 
To distinguish the two definitions, we use \emph{policy-counted} DecPOMDP as a name here, which is still a DecPOMDP with a partitioned agent set fulfilling \cref{eq:assumpt:1,eq:asst2:counting}.

\begin{definition}\label{def:decpomdp:coun}
	A \emph{policy-counted DecPOMDP} $\bar{M}$ is a tuple  $(\bar{\boldsymbol{I}},\allowbreak \bar{S},\allowbreak  \bar{\boldsymbol{A}},\allowbreak \bar{{T}},\allowbreak \bar{{R}},\allowbreak \bar{\boldsymbol{O}},\allowbreak \bar{\boldsymbol{\varOmega}}, \tau)$, with 
	\begin{itemize}
		\item $\bar{\boldsymbol{I}}$ a \emph{partitioning} $\{\mathfrak{I}_k\}_{k=1}^K$ of $N$ agents, $n_k = |\mathfrak{I}_k|$ and $|\bar{\boldsymbol{I}}| = \sum_{k} n_k = N$,
		\item $S$ a random variable with a set of states as range,
		\item $\bar{\boldsymbol{A}} = \{A_k\}_{k=1}^K$ a set of decision random variables, 
		\item ${\bar{T}(S', S, \bar{\boldsymbol{A}}_c)}$ a transition function with counted actions, 
		\item $\bar{R}(S, \bar{\boldsymbol{A}}_c)$ a reward function with counted actions, 
		\item $\bar{\boldsymbol{O}} = \{O_k\}_{k=1}^K$ a set of random variables, 
        \item $\bar{\varOmega}(\bar{\boldsymbol{O}}_c, S)$ a sensor function with counted observations, and
        \item $\tau$ a horizon.
	\end{itemize}
    where $\bar{\boldsymbol{A}}_c = \{\#_{\mathfrak{I}_k}[A_k]\}_{k=1}^K$ and $\bar{\boldsymbol{O}}_c = \{\#_{\mathfrak{I}_k}[O_k]\}_{k=1}^K$.
\end{definition}
The difference between \cref{def:decpomdp:counold} and \cref{def:decpomdp:coun} only lies in $\bar{\boldsymbol{A}}_c$ and $\bar{\boldsymbol{O}}_c$ versus $\bar{\boldsymbol{A}}$ and $\bar{\boldsymbol{O}}$, which are CRVs in the former and plain variables in the latter case, as the set of actions and observations available. 
Since \cref{eq:asst2:counting} holds in both definitions, $\bar{T}$, $\bar{R}$, and $\bar{\varOmega}$ are the same with CRVs as inputs for actions and observations.
This encoding captures the symmetric behaviour of agents within each partition and helps to reduce the model complexity by operating over counted representations instead of enumerating all individual agents. 
\Cref{thm:equiv} establishes the equivalence between the ground DecPOMDP and its counted version of \cref{def:decpomdp:coun}. 

\begin{theorem}\label{thm:equiv}
	A ground DecPOMDP $M$ that meets \cref{eq:assumpt:1,eq:asst2:counting} is equivalent to a policy-counted DecPOMDP $\bar{M}$.
\end{theorem}
The full proof can be found in \cref{app:thm2}, which is based on the fact that one can be converted into the other by construction and vice versa.
Having established the equivalence between ground and counting Dec\-POMDPs, we now show how the counting formulation affects the model complexity.

\begin{theorem} \label{thm:complexitycounting}
    For a fixed number of partitions $K$, fixed action and observation spaces, and a fixed horizon $\tau$, the model complexity of a counting DecPOMDP $\bar{M}_c$ is polynomial in the number of agents $N$. 
\end{theorem}
The full proof can be found in \cref{app:thm3}, where we show that using histograms within agent partitions keeps the model size polynomial in $N$.

\subsection{Counted Utility Calculation}
As the policy space is given by joint counted policies, we can update the semantics and $MEU$ calculation.
The semantics of the policy-counted DecPOMDP $\bar{M}$ is given by the joint counted policy space $\boldsymbol{H}^\varPi$.
The policy-counted DecPOMDP problem asks for the joint counted policy that yields the maximum expected utility (with horizon $\tau$), which is again computed by recursively calculating the expected utility over the next states and observations, until reaching $\tau$.

However, the sum over observations can now be turned into a sum over counted observations:
Given a joint counted policy, each partition is partitioned again into groups following a representative policy according to the counted partition policy, with group sizes corresponding to the histogram counts.
For a given (representative) observation history, the recursive call to the next utility as well as the remaining sub-policy are the same for these agents, making them indistinguishable among themselves.
As such, for each of these groups, we can count observations in histograms over the (sub)group with a particular representative policy and observation history, and let the sum go over the cross product of these histograms. 
Conditioning these groups on their shared observation history means that groups split up further, which can in the worst case lead to singletons if the group size is small to begin with or $\tau$ is very large but the overall complexity is still going to be only polynomial as we show afterwards. 
Additionally, we only need this deep recursion for the utility definition here and are able to define policy-counted dynamic programming using a single step. 

To formalise this setting, we define the space of representative observation histories as well as the space of observation histograms given a representative observation history. 
\begin{definition}\label{def:reprobshist}
	Given a step $t$, a \emph{representative observation history} $p_k^{0:t,o}$ of a partition $\mathfrak{I}_k$ is a sequence of observations $o^{t'} \in \ran(O_k)$ of length $t$, $t' \in \{1, \dots, t\}$, with $o^0 = \bot$ as an empty observation.
	Given a representative observation history $p_k^{0:t,o}$, let $p_k^{t',o}$ refer to the entry at point $t'$ in the sequence and $p_k^{0:t-1,o}$ refer to the subsequence without the last entry $p_k^{t,o}$.
	Finally, let $P^{0:t,o}_k$ denote the \emph{space of representative observation histories} of length $t$ for the agents in partition $\mathfrak{I}_k$, i.e., $P^{0:t,o}_k = \bot \times \times_{t'=1}^t \ran(O_k)$.
\end{definition}

\begin{definition}\label{def:obshist}
	Given a step $t$, a counted partition policy $h_k^\varPi$, and a representative policy $\pi \in \varPi_k$ for a partition $\mathfrak{I}_k$, a \emph{counted partition observation} is a set of histograms $\{h_{k,p^{0:t,o}_k}^{o,\pi}\}_{p^{0:t,o}_k \in P^{0:t,o}_k}$ where 
	\begin{align}\label{eq:obshist}
		h_{k,p^{0:t,o}_k}^{o,\pi} &= \{(o_l, n_l)\}_{o_l \in \ran(O_k)}\nonumber\\
		\intertext{with}
		\sum_{o_l \in \ran(O_k)} n_l &= 
		\begin{cases*}
			h_{k,p^{0:t-1,o}_k}^{o,\pi}(p_k^t) & if $t > 0$\\
			h_k^\varPi(\pi)              & if $t=0$
		\end{cases*}
	\end{align}
	and $p^{0:t,o}_k = p^{0:t-1,o}_k \circ p_k^t$.
	Let $H_{k}^{o,\pi}$ denote the \emph{space of counted partition observations} for a given representative policy $\pi$, i.e., the set of all possible sets of histograms that fulfil \cref{eq:obshist}.
	Given a joint counted policy $\boldsymbol{h}^\varPi$, let $H_{k}^{o,\varPi}$ denote the space of counted partition observations over all representative policies $\varPi_k$, i.e., $H_{k}^{o,\varPi} = \times_{\pi \in \varPi_k} H_{k}^{o,\pi}$, and $\boldsymbol{H}^{o,\varPi}$ the \emph{space of joint counted observations}, i.e.,  $\boldsymbol{H}^{o,\varPi} = \times_{k=1}^K H_k^{o,\varPi}$.
\end{definition}
Note that the observation histogram in \cref{eq:obshist} holds counts $n_l$ for $p^{0:t,o} \circ o_l$.
Based on these definitions, we can now rewrite \cref{meuforgroundcase} as follows:
\begin{align}
	MEU(\bar{M}) = (\operatorname*{arg\;max}_{\boldsymbol{h}^\varPi \in \boldsymbol{H}^\varPi}  U_{\bar{M}}(\boldsymbol{h}^\varPi) , \max_{\boldsymbol{h}^\varPi \in \boldsymbol{H}^\varPi}  U_{\bar{M}}(\boldsymbol{h}^\varPi))  \label{meuforcountingcase}
\end{align}
where $U_{\bar{M}}(\boldsymbol{h}^\varPi)$ is calculated recursively over $\tau$ steps and weighted according to the prior $T(s^0, .\ , .\ )$, i.e., $U_{\bar{M}}(\boldsymbol{h}^\varPi) = \sum_{s^0 \in \ran(S)} T(s_0, .\ , .\  ) U_{\bar{M}}^{\boldsymbol{h}^\varPi} (s^0, \bot^{0:0})$ with
\begin{align}\label{eq:eu:cnt}
    \begin{split}
        &U_{\bar{M}}^{\boldsymbol{h}^\varPi}(s^t, (\boldsymbol{h}^{o,\varPi})^{0:t}) = R(s^t, \boldsymbol{h}^a) + \sum_{s^{t+1} \in \ran(S)}T(s^{t+1}, s^t, \boldsymbol{h}^a) \\
        &\ \ \ \sum_{\boldsymbol{h}^{o,\varPi} \in \boldsymbol{H}^{o,\varPi}} Mul(\boldsymbol{h}^{o,\varPi}) 
         \cdot\varOmega(\boldsymbol{h}^{o},s^{t+1}) U_{\bar{M}}^{\boldsymbol{h}^\varPi} (s^{t+1}, (\boldsymbol{h}^{o,\varPi})^{0:t} \circ \boldsymbol{h}^{o,\varPi})
    \end{split}
\end{align}
with $Mul(\boldsymbol{h}^{o,\varPi})$ the multinomial coefficient, denoting the number of ground permutations encoded by the histograms \cite{TagFiDaBl13}, where for each partition $\mathfrak{I}_k$, the counted partition action $h^a_k \in \boldsymbol{h}^a$ as input to $R$ and $T$ is given by adding up, for each action $a \in \ran(A_k)$, how often $a$ is carried out over all representative policies and all representative observation histories using the counts in the latest counted partition observation if the observation history plus the newest observation map to $a$ in the current representative policy, i.e.,
\begin{align}
    \begin{split}\label{eq:cnt:ac}
    h^a_k &= \{(a_l, n_l)\}_{a_l \in \ran(A_k)} \quad \text{with}\\
    n_l &= \begin{cases*}
        \displaystyle\sum_{\substack{\pi \in\\ \varPi_k}} \sum_{\substack{p \in \\ P^{0:t-1,o}_k}} \smashoperator[r]{\sum_{\substack{o \in\\ \ran(O_k)}}} h^{o,\varPi,t}_{k,p}(o) \mathds{1}_{\pi_k(p \circ o) = a_l} & if $t > 0$ \\
        \displaystyle\sum_{\pi \in \varPi_k} h_k^\varPi(\pi) \mid \pi(\bot) = a_l & if $t = 0$
        \end{cases*}
    \end{split}
\end{align}
with $h^{o,\varPi,t}_{k,p} \in \boldsymbol{h}_k^{o,\varPi,t} \in (\boldsymbol{h}_k^{o,\varPi})^{o:t}$ and $\mathds{1}$ being the indicator function returning $1$ if $\pi_k(p \circ o) = a_l$.
The counted observation $h_k^{o} \in \ran(\#_{\mathfrak{I}_k}[O_k])$ as input to $\varOmega$ is given by adding up, for each observation $o \in \ran(O_k)$, how often $o$ is observed over all representative policies and all representative observation histories in the latest counted partition observation, i.e., \begin{align}
    \begin{split}\label{eq:cnt:obs}
    h^o_k = \{(o_l, n_l)\}_{o_l \in \ran(O_k)} \quad \text{with} \quad
    n_l = \sum_{\pi \in \varPi_k} \sum_{p \in P^{0:t-1,o}_k} h^{o,\varPi,t}_{k,p}(o_l)
    \end{split}
\end{align}

The utility calculation for a ground DecPOMDP and an equivalent policy-counted DecPOMDP is equivalent again, leading to the same maximum expected utility.
The trick is that calculations that lead to the same result are only performed once in the counted case, which is possible due to the histograms essentially encoding when calculations lead to the same result.

\begin{theorem}\label{thm:correctnesofMEU}
    The maximum expected utility for a ground DecPOMDP $M$ that meets \cref{eq:assumpt:1,eq:asst2:counting} is equal to the maximum expected utility of the corresponding policy-counted DecPOMDP $\bar{M}$.
    That is, $MEU(M) = MEU(\bar{M})$.
\end{theorem}
The full proof can be found in \cref{app:thm4}, where we show that both models give the same utility by mapping ground policies to counted policies.
Next, we look at evaluation cost and policy space size.

\begin{theorem}\label{thm:compl:pol}
	For a fixed number of partitions $K$, fixed representative action and observation spaces, and a fixed horizon $\tau$, the cost of evaluating a joint counted policy and the size of the overall joint counted policy space in a policy-counted DecPOMDP depend polynomially on the number of agents $N$.
\end{theorem}
The full proof can be found in \cref{app:thm5}, which relies on the joint counted observation space and the joint counted policy space depending polynomially on the number of agents $N$, as $n_k \leq N$ agents are distributed onto the different representative observations and policies, which can be characterised using the binomial coefficient.
Thus, we have reduced the complexity from an exponential one to a polynomial one for $N$.
As such, a policy-counted DecPOMDP is tractable in the number of agents.
\begin{corollary}
	Given a fixed number of partitions $K$, fixed representative action and observation spaces, and a fixed horizon $\tau$, a policy-counted DecPOMDP is tractable in the number of agents $N$. 
\end{corollary}
With the number of agents $N$ no longer appearing in the exponent, we are able to solve larger problem instances in terms of $N$ compared to the ground case, especially if the number of partitions $K$ is small.
Next, we use the counting encoding to speed up dynamic programming for DecPOMDPs.

\section{Policy-Counted Dynamic Programming}\label{sec:cntdp}
In this section, we extend the standard dynamic programming operator for policy-counted DecPOMDPs.
Because of the equivalence between the ground and policy-counted versions of DecPOMDPs, we can show that the value computations are equivalent, allowing for computing values at a partition level as well as pruning representative policies. 

The policy-counted dynamic programming operator still follows the three general steps of exhaustive backup, value computation, and pruning but on a partition-level.
That is, \cref{alg:dp} still applies but computes value vectors and performs pruning according to counted versions of those equations for each partition (see \cref{eq:countingf,eq:prunecounting}). 
During the exhaustive backup, the operator builds all \emph{representative} policies of depth $t$ based on the policies of depth $t-1$ for each partition and then computes their corresponding value vectors.
Since the operator builds increasingly deeper policies, the value computation stays at the top level, building its counted joint action from the root actions in the representative policies, i.e., with $t=0$ in \cref{eq:cnt:ac} to build $\boldsymbol{h}^a$, and uses empty representative histories, i.e., \cref{eq:obshist} with $t=0$ as well to build $\boldsymbol{H}^{o,\varPi}$.
Thus, for a state $s$ and joint counted policy $\boldsymbol{h}_{-i}^\varPi$ of the other agents and partitions, the value computation for a given representative policy $\pi_k$ is defined as follows
with $\pi_k \oplus \boldsymbol{h}_{-i}^\varPi$ adding $1$ to the corresponding count in $h_k^\varPi$ of $\boldsymbol{h}_{-i}^\varPi$ to form a full joint counted policy:
\begin{align}
	\begin{split} \label{eq:countingf}
		\bar{V}_k (s, \boldsymbol{h}^\varPi) = R(s,\boldsymbol{h}^a) +& \sum_{s'\in \ran(S)} T(s',s,\boldsymbol{h}^a) \\
		& \ \ \ \sum_{\boldsymbol{h}^{o,\varPi} \in \boldsymbol{H}^{o,\varPi}} Mul(\boldsymbol{h}^{o,\varPi}) \cdot \varOmega(\boldsymbol{h}^o, s') \bar{V}_k (s', \boldsymbol{h}^\varPi.\boldsymbol{h}^{o,\varPi}) 
	\end{split}
\end{align}
where $\boldsymbol{h}.\boldsymbol{h}^{o,\varPi}$ denotes the policy histogram that follows from the observation histogram $\boldsymbol{h}^{o,\varPi}$.
Specifically, given depth $t$ of the representative policies $\Pi_k^{t}$ used for $\boldsymbol{h}^\varPi$, the joint counted policy $\boldsymbol{h}^\varPi$, and the current joint counted observation $\boldsymbol{h}^{o,\varPi}$, the new policy histogram is calculated by adding up how often a subpolicy $\pi_l \in \Pi_k^{t-1}$ remains over the different representative policies $\pi \in \Pi_k^t$ when following the different observations $o \in \ran(O_k)$ given the counts in $h_k^{o,\varPi} \in \boldsymbol{h}^{o,\varPi}$ for $o$, i.e.,
\begin{align*}
    \begin{split}
        \boldsymbol{h}.\boldsymbol{h}^{o,\varPi} = \{(\pi_l, n_l)\}_{\pi_l \in \Pi_k^{t-1}} \quad \text{with} \quad
        n_l = \sum_{\pi \in \Pi_k^{t}} \sum_{o \in \ran(O_k)} {h}^{o,\varPi}_k(o) \mid \pi.o = \pi_l
    \end{split}
\end{align*}
During pruning, the operator eliminates weakly dominated representative policies $\pi_k$ based on the following equation (solvable by a linear programme, see \cref{app:lp:cnt}):
\begin{align}
	\begin{split}
		forall s\in \ran(S), \boldsymbol{h}_{-i}^{t} \in \boldsymbol{H}_{-i}^t \exists {\pi}_{l \neq k}^{t} \in \varPi_k^{t} : 
		\bar{V}_k(s, {\pi}_{k}^{t} \oplus \boldsymbol{h}_{-i}^{t}) \leq \bar{V}_k(s,{\pi}_{l}^{t} \oplus \boldsymbol{h}_{-i}^{t}) \label{eq:prunecounting}
	\end{split}
\end{align}
\Cref{eq:prunecounting} ensures that, for each state and each counted policy of the other partitions, a representative policy is only pruned if there exists another one with equal or higher value.

Next, we show that counted dynamic programming using \cref{eq:countingf,eq:prunecounting} is equivalent to dynamic programming using \cref{eq:valuef,eq:prune}, which is a culmination of the formalisations and theorems so far.
\begin{theorem} \label{thm:theo6}
	Using policy-counted dynamic programming on a policy-counted DecPOMDP $\bar{M}$ is equivalent to using standard dynamic programming on a ground DecPOMDP $M$, in which \cref{eq:assumpt:1,eq:asst2:counting} hold.
\end{theorem}
The full proof can be found in App.~\ref{app:thm6}, where we show that each step of the dynamic programming operator works equivalently in both the ground and counted models.
Given the equivalence of the policy-counted and standard version of dynamic programming, correctness of the policy-counted version follows directly from the correctness of the standard version.
\begin{corollary}
	Policy-counted dynamic programming using \cref{eq:countingf,eq:prunecounting} is correct.
\end{corollary}
Next, we show that policy-counted dynamic programming depends polynomially on the number of agents assuming a fixed and small number of representative actions and observations and partitions as well as horizon.

\begin{theorem} \label{thm:theo7}
	For a fixed number of partitions $K$, fixed representative action and observation spaces, and a fixed horizon $\tau$, the runtime of policy-counted dynamic programming in a policy-counted DecPOMDP $\bar{M}$ depends polynomially on the number of agents $N$.
\end{theorem}
The full proof can be found in App.~\ref{app:thm7}, which analyzes the polynomial runtime of the backup, value calculation, and pruning steps.

\Cref{app:tiger} shows the well-known DecTiger model \cite{NaiTaYoPyMa03} as ground, counting, and policy-counted DecPOMDPs to highlight the similarities and differences between the models.
In summary, policy-counted dynamic programming is able to use the compact encoding of policy-counted DecPOMDPs to generate a solution that is equivalent to the ground solution in time polynomial in $N$ instead of exponential in $N$, assuming all other parameters to be fixed.
As the DecTiger example shows, the pay-off only arises if $N$ is sufficiently large, i.e., if the assumption of $N \ll K$ indeed holds.
Otherwise the counting of policies introduces too much overhead.
Future work is left to deal with the remaining exponent that can still become rather large.

\section{Conclusion}
This paper presents policy-counted DecPOMDPs, which exploit a partitioned agent set under certain structural assumptions to reduce the complexity from exponential to polynomial in the number of agents, assuming all other parameters to be fixed.
Specifically, we are able to use representative policies per partition, enabling counting how many agents follow which policy, which allows for a compact representation and considerably reduces the model's size, computing cost, and policy space, leading to tractability in agent numbers assuming a fixed number of partitions, horizon, and representative action and observation spaces.
Additionally, we present a policy-counted dynamic programming for policy-counted DecPOMDPs that updates policies and values as well as prunes policies directly at the partition level. 
As a consequence, policy-counted DecPOMDPs offer a lot of potential for large-scale applications such as nanoscale medical systems. 

In future research, we aim to further enhance scalability by working on policy encodings as well as applying lifting to large state spaces, to enable additional complexity reductions.
Another idea is to look at approximations, considering that histograms are rarely needed in count steps of $1$.

\bibliographystyle{splncs04}
\bibliography{lifted_inf_utf8}

\begin{thebibliography}{10}
\providecommand{\url}[1]{\texttt{#1}}
\providecommand{\urlprefix}{URL }
\providecommand{\doi}[1]{https://doi.org/#1}

\bibitem{BraGeLaMo22}
Braun, T., Gehrke, M., Lau, F., M{\"o}ller, R.: {Lifting in Multi-agent Systems
  under Uncertainty}. In: UAI-22 Proc. of the 38th Conference on Uncertainty in
  Artificial Intelligence. pp.~1--8. AUAI Press (2022)

\bibitem{BraMo16a}
Braun, T., M{\"o}ller, R.: {Lifted Junction Tree Algorithm}. In: Proc. of {KI}
  2016: Advances in Artificial Intelligence. pp. 30--42. Springer (2016)

\bibitem{BroTaMeDaRa11}
Van~den Broeck, G., Taghipour, N., Meert, W., Davis, J., De~Raedt, L.: {Lifted
  Probabilistic Inference by First-order Knowledge Compilation}. In: IJCAI-11
  Proc. of the 22nd International Joint Conference on Artificial Intelligence.
  pp. 2178--2185. IJCAI Organization (2011)

\bibitem{CarJaBhSpTo21}
Carr, S., Jansen, N., Bharadwaj1, S., Spaan, M.T.J., Topcu, U.: Safe policies
  for factored partially observable stochastic games. In: RSS-21 Proc. of
  Robotics: Science and Systems XVII. pp. 1--11. RSS Foundation (2021)

\bibitem{EmeGoScTh04}
Emery-Montemerlo, R., Gordon, G., Schneider, J., Thrun, S.: Approximate
  solutions for partially observable stochastic games with common payoffs. In:
  AAAMAS-04 Proc. of the 3rd International Joint Conference on Autonomous
  Agents and Multiagent Systems. pp. 136--143. IEEE (2004)

\bibitem{HanBeZu04}
Hansen, E.A., Bernstein, D.S., Zilberstein, S.: Dynamic programming for
  partially observable stochastic games. In: AAAI-04 Proc. of the 19th National
  Conference on Artificial Intelligence. vol.~4, pp. 709--715 (2004)

\bibitem{HorBo19}
Hor{\'a}k, K., Bo{\v{s}}ansk{\`y}, B.: Solving partially observable stochastic
  games with public observations. In: Proc. of the AAAI conference on
  Artificial Intelligence. pp. 547--552. AAAI Press (2019)

\bibitem{HorBoKiKa19}
Hor{\'a}k, K., Bo{\v{s}}ansk{\`y}, B., Kiekintveld, C., Kamhoua, C.: {Compact
  Representation of Value Function in Partially Observable Stochastic Games}.
  In: IJCAI-19 Proc. of the 28th International Joint Conference on Artificial
  Intelligence. pp. 350--356. IJCAI Organisation (2019)

\bibitem{HorBoPe17}
Hor{\'a}k, K., Bo{\v{s}}ansk{\`y}, B., P{\v{e}}chou{\v{c}}ek, M.: {Heuristic
  Search Value Iteration for One-sided Partially Observable Stochastic Games}.
  In: AAAI-17 Proc. of the 31st AAAI Conference on Artificial Intelligence. pp.
  558--564 (2017)

\bibitem{karabulut2024lifting}
Karabulut, N.N., Braun, T.: Lifting partially observable stochastic games. In:
  International Conference on Scalable Uncertainty Management. pp. 201--216.
  Springer (2024)

\bibitem{karabulut2025counting}
Karabulut, N.N., Braun, T.: Counting agents in partially observable stochastic
  games. In: European Conference on Symbolic and Quantitative Approaches with
  Uncertainty. pp. 207--222. Springer (2025)

\bibitem{koops2023recursive}
Koops, W., Jansen, N., Junges, S., Simao, T.D.: {Recursive Small-step
  Multi-agent A* for Dec-POMDPs}. In: IJCAI-23 Proceedings of the 32nd
  International Joint Conference on Artificial Intelligence. pp. 5402--5410.
  IJCAI Organization (2023)

\bibitem{koops2024approximate}
Koops, W., Junges, S., Jansen, N.: {Approximate Dec-POMDP Solving Using
  Multi-agent A*}. In: IJCAI-24 Proceedings of the 33nd International Joint
  Conference on Artificial Intelligence. pp. 6743--6751. IJCAI Organization
  (2024)

\bibitem{KumZi09}
Kumar, A., Zilberstein, S.: Dynamic programming approximations for partially
  observable stochastic games. In: FLAIRS-09 Proc. of the 22nd International
  Florida Artificial Intelligence Research Society Conference. AAAI Press
  (2009)

\bibitem{MarBrMoGe26}
Marwitz, F., Braun, T., M{\"o}ller, R., Gehrke, M.: {Lifted Forward Planning in
  Relational Factored Markov Decision Processes with Concurrent Actions}. In:
  AAMAS-26 Proc. of the 25th International Conference on Autonomous Agents and
  Multiagent Systems (2026), \url{https://arxiv.org/pdf/2505.22147v2}

\bibitem{MilZeKeHaKa08}
Milch, B., Zettelmoyer, L.S., Kersting, K., Haimes, M., Kaelbling, L.P.:
  {Lifted Probabilistic Inference with Counting Formulas}. In: AAAI-08 Proc. of
  the 23rd AAAI Conference on Artificial Intelligence. pp. 1062--1068. AAAI
  Press (2008)

\bibitem{NaiTaYoPyMa03}
Nair, R., Tambe, M., Yokoo, M., Pynadath, D.V., Marsella, S.: {Taming
  Decentralized POMDPs: Towards Efficient Policy Computation for Multiagent
  Settings}. In: IJCAI-03 Proc. of the 18th International Joint Conference on
  Artificial Intelligence. pp. 705--711. IJCAI Organization (2003)

\bibitem{NieBr14}
Niepert, M., Van~den Broeck, G.: {Tractability through Exchangeability: A New
  Perspective on Efficient Probabilistic Inference}. In: AAAI-14 Proc. of the
  28th AAAI Conference on Artificial Intelligence. pp. 2467--2475. AAAI Press
  (2014)

\bibitem{OliAm16}
Oliehoek, F.A., Amato, C.: {A Concise Introduction to Decentralised POMDPs}.
  Springer (2016)

\bibitem{oliehoek2009lossless}
Oliehoek, F.A., Whiteson, S., Spaan, M.T.: {Lossless Clustering of Histories in
  Decentralized POMDPs}. In: AAMAS-09 Proceedings of the 8th International
  Conference on Autonomous Agents and Multiagent Systems. pp. 577--584. IFAAMAS
  (2009)

\bibitem{Poo03}
Poole, D.: {First-order Probabilistic Inference}. In: IJCAI-03 Proc. of the
  18th International Joint Conference on Artificial Intelligence. pp. 985--991.
  IJCAI Organization (2003)

\bibitem{RusNo21}
Russell, S., Norvig, P.: {Artificial Intelligence: A Modern Approach}. Pearson
  (2021)

\bibitem{SanBo09}
Sanner, S., Boutilier, C.: {Practical Solution Techniques for First-order
  MDPs}. Artificial Intelligence Journal  \textbf{173},  748--788 (2009)

\bibitem{SanKe10}
Sanner, S., Kersting, K.: {Symbolic Dynamic Programming for First-order
  POMDPs}. In: AAAI-10 Proc. of the 24th AAAI Conference on Artificial
  Intelligence. pp. 1140--1146. AAAI Press (2010)

\bibitem{seuken2007memory}
Seuken, S., Zilberstein, S.: {Memory-Bounded Dynamic Programming for
  DEC-POMDPs}. In: IJCAI-07 Proceedings of the 21st International Joint
  Conference on Artificial Intelligence. pp. 2009--2015. IJCAI Organization
  (2007)

\bibitem{szer2006point}
Szer, D., Charpillet, F.: Point-based dynamic programming for dec-pomdps. In:
  AAAI. vol.~6, pp. 1233--1238 (2006)

\bibitem{szer2012maa}
Szer, D., Charpillet, F., Zilberstein, S.: {MAA*: A Heuristic Search Algorithm
  for Solving Decentralized POMDPs}. In: UAI-05 Proceedings of the 21st
  Conference on Uncertainty in Artificial Intelligence. pp. 576--583. ACM
  (2005)

\bibitem{TagFiDaBl13}
Taghipour, N., Fierens, D., Davis, J., Blockeel, H.: {Lifted Variable
  Elimination: Decoupling the Operators from the Constraint Language}. Journal
  of Artificial Intelligence Research  \textbf{47}(1),  393--439 (2013)

\bibitem{TomHoArBoCh21}
Tom{\'a}{\v{s}}ek, P., Hor{\'a}k, K., Aradhye, A., Bo{\v{s}}ansk{\`y}, B.,
  Chatterjee, K.: Solving partially observable stochastic shortest-path games.
  In: IJCAI-21 Proc. of the 30th International Joint Conference on Artificial
  Intelligence. pp. 4182--4189. IJCAI Organisation (2021)

\bibitem{vaidya1989speeding}
Vaidya, P.M.: {Speeding-up Linear Programming Using Fast Matrix
  Multiplication}. In: 30th Annual Symposium on Foundations of Computer
  Science. pp. 332--337. IEEE Computer Society (1989)

\bibitem{WigOlRo16}
Wiggers, A.J., Oliehoek, F.A., Roijers, D.M.: Structure in the value function
  of two-player zero-sum games of incomplete information. In: ECAI-16 Proc. of
  the 22nd European Conference on Artificial Intelligence. pp. 1628--1629. IOS
  Press (2016)

\end{thebibliography}

\newpage

\title{The Curious Case of Exploding DecPOMDPs:\\ Containing the Fire through Policy Counting\\(Supplementary Material)}
\titlerunning{The Curious Case of Exploding DecPOMDPs (Supp. Mat.)}
\author{}
\authorrunning{N.\ N.\ Karabulut and T. Braun}
%
\institute{}
\maketitle

\appendix
\renewcommand{\theHsection}{app.\thesection}
\section{Linear Programmes}
\label{app:lp}
This section provides the linear programmes used during the pruning step in dynamic programming, which repeats the pruning condition from the main paper and then lists the linear programme.

\subsection{Ground DecPOMDP} \label{app:lp:gr}

Formally, a policy $\pi_{i,j}^{t}$ is pruned if the following holds, 
\begin{align*}
    \begin{split}
	\forall s\in \ran(S), \boldsymbol{\pi}_{-i}^{t} \in \boldsymbol{\varPi}_{-i}^t \exists {\pi}_{i,j'}^{t} \in \varPi_i^{t} : 
    V_i'(s, {\pi}_{i,j}^{t} \circ \boldsymbol{\pi}_{-i}^{t}) \leq V'_i(s,{\pi}_{i,j'}^{t} \circ \boldsymbol{\pi}_{-i}^{t}) 
    \end{split}
\end{align*}
which can be solved by solving the following linear programme (prune $\pi_{i,j}^{t}$ if $d < 0$):
 {\allowdisplaybreaks
 \begin{align*} 
 	\text{variables: }		& b_i(s,\boldsymbol{\pi}_{-i}), d\\
 	\text{maximise: }		& d\\
 	\text{constraints: }	& \forall {\pi}_{i,j'}^{t} \in \varPi_i^{t}\\
 		 					&\ \ 	\sum_{s\in \ran(S)} \sum_{\boldsymbol{\pi}_{-i} \in \boldsymbol{\varPi}_{-i}^t} b_i(s, \boldsymbol{\pi}_{-i}) V_i(s,{\pi}_{i,j}^{t},\boldsymbol{\pi}_{-i}^{t}) \\
 							&\ \ \ \ - \sum_{s\in \ran(S)} \sum_{\boldsymbol{\pi}_{-i} \in \varPi_{-i}^t} b_i(s, \boldsymbol{\pi}_{-i}) V_i(s,{\pi}_{i,j'}^{t},\boldsymbol{\pi}_{-i}^{t}) - d \leq 0\\
 							& \sum_{s\in \ran(S)} \sum_{\boldsymbol{\pi}_{-i} \in \varPi_{-i}^t} b_i(s, \boldsymbol{\pi}_{-i}) = 1, b_i(s, \boldsymbol{\pi}_{-i}) > 0
 \end{align*}
 }

\subsection{Policy-Counted DecPOMDP}\label{app:lp:cnt}
Formally, a representative policy $\pi_{k}^{t}$ is pruned if the following holds, which can still be solved using a linear program:
\begin{align*}
    \begin{split}
	\forall s\in \ran(S), \boldsymbol{h}_{-i}^{t} \in \boldsymbol{H}_{-i}^t \exists {\pi}_{l \neq k}^{t} \in \varPi_k^{t} : 
    \bar{V}_k' (s, {\pi}_{k}^{t} \oplus \boldsymbol{h}_{-i}^{t}) \leq \bar{V}'_k(s,{\pi}_{l}^{t} \oplus \boldsymbol{h}_{-i}^{t}),
    \end{split}
\end{align*}
 which can be solved by solving the following linear programme (prune $\pi_{k,l}^{t}$ if $d < 0$):
 {\allowdisplaybreaks
 \begin{align*}
 	\text{variables: }		& b_k(s, \boldsymbol{h}_{-i}), d\\
 	\text{maximise: }		& d\\
 	\text{constraints: }	& \forall {\pi}_{k,l'}^{t} \in \varPi_k^{t}\\
 		 					&\ \ 	\sum_{s\in \ran(S)} \sum_{\boldsymbol{h}_{-i} \in \boldsymbol{H}_{-i}^t} b_k(s, \boldsymbol{h}_{-i}) V_k(s,{\pi}_{k,l}^{t} \oplus \boldsymbol{h}_{-i}^{t}) \\
 							&\ \ \ \ - \sum_{s\in \ran(S)} \sum_{\boldsymbol{h}_{-i} \in \boldsymbol{H_{-i}^t}} b_k(s, \boldsymbol{h}_{-i}) V_k(s,{\pi}_{k,l'}^{t} \oplus \boldsymbol{h}_{-i}^{t}) - d \leq 0\\
 							& \sum_{s\in \ran(S)} \sum_{\boldsymbol{h}_{-i} \in \boldsymbol{H_{-i}^t}} b_k(s, \boldsymbol{h}_{-i}) = 1, b_k(s, \boldsymbol{h}_{-i}) > 0
 \end{align*}
 }

\section{Full Proofs}
\label{app:proofs}

\subsection{Equivalence between a Ground DecPOMDP and a Counting DecPOMDP}\label{app:countingequiv}
This theorem shows equivalence between a ground DecPOMDP and a counting DecPOMDP based on \cref{eq:assumpt:1,eq:asst2:counting}.
\begin{theorem}
    A ground DecPOMDP $M$ that meets \cref{eq:assumpt:1,eq:asst2:counting} is equivalent to a counting DecPOMDP $\bar{M}_c$. 
\end{theorem}

\begin{proof}
    To convert a DecPOMDP $M = (\boldsymbol{I},S,\boldsymbol{A},\boldsymbol{O},T,R,\varOmega, \tau)$ into a counting DecPOMDP $\bar{M}_c = (\bar{\boldsymbol{I}},S,\bar{\boldsymbol{A}}_c,\bar{\boldsymbol{O}}_c,\bar{T},\bar{R}, \bar{\varOmega}, \tau)$, we apply \cref{eq:assumpt:1,eq:asst2:counting}.
    First, agents within the same partition share the same action and observation spaces, meaning that for any $i,j \in \mathfrak{I}_k$, $A_i=A_j$ and $O_i=O_j$, which is provided by \cref{eq:assumpt:1}.
    For each partition $\mathfrak{I}_k$, we introduce the CRVs $\#_{\mathfrak{I}_k}[A_k]$ and $\#_{\mathfrak{I}_k}[O_k]$ for actions and observations, whose ranges are all histograms over the local action space $\ran(A_k)$ and observation space $\ran(O_k)$. 
    To build the transition, sensor, and reward functions, based on \cref{eq:asst2:counting}, we replace $(A_i)_{i \in \mathfrak{I}_k}$ and $(O_i)_{i \in \mathfrak{I}_k}$ with CRVs $\#_{\mathfrak{I}_k}[A_k]$ and $\#_{\mathfrak{I}_k}[O_k]$ in the inputs of the transition, reward, and sensor function for each partition, with a histogram $h_k$ mapping to $\rho$ for all permutations of a partition input $\boldsymbol{a}_k$ or $\boldsymbol{o}_k$ mapping to $\rho$ for which $h$ is the histogram representation.\\
    To convert a counting DecPOMDP into a ground DecPOMDP fulfilling \cref{eq:assumpt:1,eq:asst2:counting}, we essentially reverse the steps above:
    The agent set $\boldsymbol{I}$ is the union of the partitions $\bigcup_{k=1}^K \mathfrak{I}_k$.
    The action and observation sets are given by $\{A_{k,i}\}_{i \in \mathfrak{I}_k}$ and $\{O_{k,i}\}_{i \in \mathfrak{I}_k}$, fulfilling \cref{eq:assumpt:1}.
    In the transition, reward, and sensor function, the CRV inputs are replaced by the action and observation sets, with a histogram $h$ mapping to $\rho$ being replaced by a set of mappings over all permutations of the values that agents can take according to the counts in $h$, all mapping to $\rho$, thereby leading to the functions fulfilling \cref{eq:asst2:counting}.
\end{proof}

\subsection{Full Proof for Theorem 2}\label{app:thm2}
We repeat \cref{thm:equiv} and then provide its proof.
\begin{theorem}
    A ground DecPOMDP $M$ that meets \cref{eq:assumpt:1,eq:asst2:counting} is equivalent to a counting DecPOMDP $\bar{M}$.
\end{theorem}
\begin{proof}
    First, we convert a DecPOMDP $M = (\boldsymbol{I},\allowbreak S,\allowbreak \boldsymbol{A},\allowbreak \boldsymbol{O},\allowbreak T,\allowbreak {R},\allowbreak \boldsymbol{\varOmega}, \tau)$, in which \cref{eq:assumpt:1,eq:asst2:counting} hold, into a counting DecPOMDP $\bar{M} = (\bar{\boldsymbol{I}},\allowbreak \bar{\boldsymbol{S}},\allowbreak  \bar{\boldsymbol{A}},\allowbreak \bar{\boldsymbol{O}},\allowbreak \bar{{T}},\allowbreak \bar{{R}}, \boldsymbol{\varOmega}, \tau)$. 
    Agents that share the same action and observation spaces, meaning $A_i = A_j$ and $O_i = O_j$ for $i, j \in \boldsymbol{I}$ by \cref{eq:assumpt:1}, as well as exhibit symmetric behaviour as in \cref{eq:assumpt:2} are grouped into a partition, storing for each partition a decision variable $A_k = A_i = A_j$ and an observation variable $O_k = O_i = O_j$.
    To construct the counted functions $\bar{T}$, $\bar{R}$, and $\bar{\varOmega}$ based on \cref{eq:asst2:counting}, we replace $(A_i)_{i \in \mathfrak{I}_k}$ with a CRV $\#_{\mathfrak{I}_k}[A_k]$ in the inputs of the transition, reward, and sensor function for each partition. 
    We do the same with the observation variables in the sensor function.
    For each function and each partition, we then construct from a mapping with a partition action $\boldsymbol{a}_k$ (partition observation $\boldsymbol{o}_k$) mapping to an outcome $\rho$ a new mapping with a histogram $h_k$ replacing $\boldsymbol{a}_k$ ($\boldsymbol{o}_k$), counting how often each $a \in \ran(A_k)$ ($o \in \ran(O_k)$) occurs in $\boldsymbol{a}_k$ ($\boldsymbol{o}_k$) for the counts in $h_k$, and discard all mappings with permutations of $\boldsymbol{a}_k$ ($\boldsymbol{o}_k$).\\
    Second, we convert a counting DecPOMDP into a DecPOMDP $M$ that meets \cref{eq:assumpt:1,eq:asst2:counting} by setting $\boldsymbol{I} = \bigcup_{k} \mathfrak{I}_k$, $A_i = A_k$ and $O_i = O_k$ for all $i \in \mathfrak{I}_k$ in each partition $\mathfrak{I}_k$, making $M$ fulfil \cref{eq:assumpt:1}, and then essentially expanding the counted functions into their ground versions, i.e., replacing each CRV $\#_{\mathfrak{I}_k}[V_k]$, $V \in \{A,O\}$, with a set of variables $V_{k,1}, \dots, V_{k,n_k}$ as inputs and mapping the different inputs $\boldsymbol{v}$ to an outcome $\rho$, whenever the histogram representation of $\boldsymbol{v}$ maps to $\rho$, which makes $M$ meet \cref{eq:asst2:counting}.
\end{proof}

\subsection{Full Proof of Theorem 3} \label{app:thm3}
We repeat \cref{thm:complexitycounting} and then provide its proof.
\begin{theorem} 
    For a fixed number of partitions $K$, fixed action and observation spaces, and a fixed horizon $\tau$, the model complexity of a counting DecPOMDP $\bar{M}_c$ is polynomial in $N$. 
\end{theorem}
\begin{proof}
    The histogram space of a CRV for partition $\mathfrak{I}_k, k \in \{1, \dots, K\}$ of size $n_k$ with $m$ possible values has size $\binom{n_k+m-1}{m-1} \leq n_k^{m}$ \cite{MilZeKeHaKa08}.
    As we assume $K$ is fixed and $K \ll N$, $n_k$ and $N$ have the same order of magnitude, i.e., $n_k \leq N$.
    As such, due to the range sizes of their inputs, the sizes of the functions $\bar{T}$, $\bar{R}$ and $\bar{\varOmega}$ lie in
        $O(s^2 N^{Ka})$, 
        $O(s N^{Ka})$, and
        $O(s N^{Ko})$, 
    respectively, with $s = |\ran(\bar{S})|$, $a = \max_{k}|\ran(A_k)|$, and $o =\allowbreak \max_{k} |\ran(O_k)|$, depending polynomially on $N$.
\end{proof}

\subsection{Full Proof of Theorem 4} \label{app:thm4}
We repeat \cref{thm:correctnesofMEU} and then provide its proof.
\begin{theorem}
    The maximum expected utility for a ground DecPOMDP $M$ that meets \cref{eq:assumpt:1,eq:asst2:counting} is equal to the maximum expected utility of the corresponding policy-counted DecPOMDP $\bar{M}$.
    That is, $MEU(M) = MEU(\bar{M})$.
\end{theorem}
\begin{proof}
    Let $\varPi_M$ be set of possible joint policies in the ground model $M$.
    In the policy-counted model $\bar{M}$, we convert any ground joint policy $\boldsymbol{\pi}=(\pi_1 , ... ,\pi_N)$ into a joint counted policy of histograms $\boldsymbol{h}^\varPi=(h_1, ...,h_K)$, where each $h_k$ counts the number of agents in partition $\mathfrak{I}_k$ assigned to each representative policy in $\varPi_k$.
    To establish value equivalence, consider any ground policy $\boldsymbol{\pi}$ and its corresponding policy histogram $\boldsymbol{h}^\varPi$ and an empty joint (counted observation) at the start.
    Based on \cref{thm:equiv}, the reward, transition, and observation functions return the same values between the two models given equivalent inputs.
    As a result, the immediate rewards in \cref{eq:eu:gr,eq:eu:cnt} are identical, i.e., $R(s, \boldsymbol{a}) = \bar{R}(s, \boldsymbol{h}^a)$, with the joint action in \cref{eq:eu:gr} and the joint counted action in \cref{eq:eu:cnt} being equivalent, since the same policies are followed between the ground and the policy-counted version and thus, the actions in the roots are identical, adding up to the counts in the counted joint action.
    The following sum over states is identical in both equations with the transition functions returning the same values.
    The sum over joint (counted) observations is equivalent, since every ground joint observation can be translated into a joint counted observation, with the counted version combining those ground joint observations with the same counted representation into one summand and multiplying that summand with the number of instances using the multinomial coefficient \cite{TagFiDaBl13}.
    The summand value is identical, since the sensor functions return the same value for equivalent inputs.
    As such, the calculations at the end of the first step are equivalent.\\
    Subsequent utility calculations have as the observation history an equivalent input as just argued.
    Thus, with the same arguments as before, the further joint (counted) actions are equivalent, with the action histogram being a counted representation of the ground joint action, which also holds for further joint (counted) observations.
    Therefore, the expected utilities $U_M(\boldsymbol{\pi})$ and $U_{\bar{M}}(\boldsymbol{h}^\varPi)$ are identical.
    As this can be done for every ground policy with its corresponding policy histogram, the counted policy space $\boldsymbol{H}^\varPi$ and the ground policy space $\varPi_M$ yield the same maximum expected utility, that is $MEU(M) = MEU(\bar{M})$. 
\end{proof} 

\subsection{Full Proof of Theorem 5}\label{app:thm5}
We repeat \cref{thm:compl:pol} and then provide its proof.

\begin{theorem}
    For a fixed number of partitions $K$, fixed action and observation spaces, and a fixed horizon $\tau$, the cost of evaluating a joint counted policy and the size of the overall joint counted policy space in a policy-counted DecPOMDP depend polynomially on the number of agents $N$.
\end{theorem}

\begin{proof}
    For the cost of evaluating a joint counted policy, consider \cref{eq:eu:cnt} with a sum over states, which is of size $s = |\ran(S)|$, and a sum over joint counted observations.
    With $a = \max_{k} |\ran(A_k)|, o = \max_{k} |ran (O_k)|$, the space of joint counted observations is structured for each partition by the number of representative policies $p$ and the number of representative observation histories $h$, which are given by \cite{OliAm16}
    \begin{align*}
        p & = |\varPi_k|= a^{\frac{o^\tau-1}{o-1}}  \\ 
        h &= |P_k^{0:\tau,o}| = {o^{\tau}}  
    \end{align*}
    Each representative history can have up to $n_k \leq N$ agents following that particular history with $K \ll N$, which are distributed onto $o$ possible observations.
    Thus, the number of histograms per partition is given by 
    \begin{align*}
        H = \binom{N + o - 1}{o - 1} \leq N^o
    \end{align*}
    Therefore, the size of the joint counted observation space is given by the size of the cross product of $K$ partitions, $p$ representative policies, and $h$ representative histories, i.e., $H^{Kph}$, which is capped by
    \begin{align*}
        H^{Kph} \leq (N^o)^{Ka^{o^\tau}{o^\tau}} = N^{{Kao^{2\tau+1}}} 
    \end{align*}
    and as such, depends polynomially on $N$.
    The size of the joint counting policy space is capped by
    \begin{align*}
        \prod_{k=1}^{K} \binom{n_k + p - 1}{p -1} \leq N^{Kp} = N^{Ka^{\frac{o^\tau-1}{o-1}}}
    \end{align*}
    and thereby, also depends polynomially on $N$.
\end{proof}

\subsection{Full Proof of Theorem 6}\label{app:thm6}
We repeat \cref{thm:theo6} and then provide its proof.
\begin{theorem} 
    Using policy-counted dynamic programming on a policy-counted DecPOMDP $\bar{M}$ is equivalent to using standard dynamic programming on a ground DecPOMDP $M$, in which \cref{eq:assumpt:1,eq:asst2:counting} hold.
\end{theorem}
\begin{proof}
    We show equivalence by considering the different steps of the dynamic programming operator.\\
    \emph{Exhaustive backup: }
    Given \cref{eq:assumpt:1}, the policies for each agent in a partition $\mathfrak{I}_k$ are identical, meaning that the exhaustive backup can be performed using representative policies for each partition.\\
    \emph{Value computation: }
    Given that $MEU(\bar{M}) = MEU(M)$ by \cref{thm:correctnesofMEU} and as such $U^{\boldsymbol{h}^\varPi}_{\bar{M}} = U^{\boldsymbol{\pi}}_M$ for equivalent policies $\boldsymbol{h}^\varPi$ and $\boldsymbol{\pi}$, $\bar{V}'_k = V'_i$ for all $i \in \mathfrak{I}_k$ as the value computation is a rewriting of the utility computation, meaning that the value calculations can be performed for counted partition policies at each partition.\\
    \emph{Pruning: }
    Since $\bar{V}'_k = V'_i$ for all $i \in \mathfrak{I}_k$, \cref{eq:prune} for pruning is identical for all $i \in \mathfrak{I}_k$ given a representative policy $\pi_k \in \Pi_k$, meaning that $\pi_k$ can be checked for pruning once per partition and thus, pruning can be performed for representative policies per partition.\\
    Given the equivalence of the three steps, we can conclude that the overall procedure yields an equivalent result.
\end{proof}

\subsection{Full Proof of Theorem 7} \label{app:thm7}
We repeat \cref{thm:theo7} and then provide its proof.
\begin{theorem} 
    For a fixed number of partitions $K$, fixed action and observation spaces, and a fixed horizon $\tau$, the runtime of policy-counted dynamic programming in a policy-counted DecPOMDP $\bar{M}$ depends polynomially on the number of agents $N$.
\end{theorem}
\begin{proof}
    Again, we consider the three steps of the dynamic programming operator to show polynomial dependence on the number of agents $N$, which reuses definitions of \cref{thm:complexitycounting}.\\
    \emph{Exhaustive backup: }
    With $a = \max_{k} |\ran(A_k)|, o = \max_{k} |ran (O_k)|$, the number of policies to generate per partition is given by $p = a^{\frac{o^\tau-1}{o-1}}$ \cite{OliAm16}, meaning $Kp$ over all partitions,  which is independent of $N$.\\
    \emph{Value calculation: }
    \cref{eq:countingf} depends polynomially on $N$ since the sum over joint counted observations depends polynomially on $N$.
    In contrast to the recursive calculation of $U^{\boldsymbol{h}^\varPi}_{\bar{M}}$, we do not need to consider representative observation histories for $t>0$, meaning that the size of the space of histograms is capped by $(N^o)^{Kp} = N^{\frac{Koa(o^\tau -1)}{o-1}}$, which is polynomial in $N$.\\
    \emph{Pruning: }
    Solving the pruning part can be done using a linear programme, which can be solved in time polynomial w.r.t.\ the number of variables and constraints \cite{vaidya1989speeding}, whose numbers again depends polynomially on $N$, since the number of variables and constraints depends on the counted policy space $\boldsymbol{H}^\varPi$, which depends polynomially on $N$.\\
    Given that each step depends polynomially on $N$, the overall runtime depends polynomially on $N$.
\end{proof}

\section{DecTiger Example}
\label{app:tiger}
We provide the different model definitions, look at model and worst-case policy space sizes, and consider the differences when applying the dynamic programming operator.

\subsection{Model Definition}
We use the specification of the DecTiger benchmark from the MADP tool box. \footnote{\url{https://github.com/MADPToolbox/MADP/blob/master/problems/dectiger.dpomdp}}.
\Cref{lst:dectiger} shows the original DecTiger version in the MADP input format. 
The DecPOMDP model reads as follows:
\begin{itemize}[noitemsep]
	\item $\boldsymbol{I} = \{agent_1, agent_2\}$, 
	\item $S$, $\ran(S) = \{$\emph{tiger-left}, \emph{tiger-right}$\} = \{tl,tr\}$,
	\item $\boldsymbol{A} = \{A_i\}_{i\in \boldsymbol{I}}$, $\forall i \in \boldsymbol{I}: \ran(A_i) = \{$\emph{listen}, \emph{open-left}, \emph{open-right}$\} = \{li,ol,or\}$, and
	\item $\boldsymbol{O} = \{O_i\}_{i \in \boldsymbol{I}}$, $\forall i \in \boldsymbol{I} : \ran(O_i) = \{$\emph{hear-left}, \emph{hear-right}$\} = \{hl,hr\}$. 
\end{itemize}
with $T$, $R$, and $\varOmega$ (= O in \cref{lst:dectiger}) in tabular notation in full below (lines are reordered compared to the listing to match the order of the inputs in the definitions). 
\begin{center}
\begin{tabular}{cccll}
    \toprule
    $S$   & $A_1$   & $A_2$ & $T_{tl}$  & $T_{tr}$ \\
    \midrule
    $tl$  & $li$    & $li$  & $1$       & $0$   \\
    $tl$  & $li$    & $ol$  & $0.5$     & $0.5$ \\
    $tl$  & $li$    & $or$  & $0.5$     & $0.5$ \\
    $tl$  & $ol$    & $li$  & $0.5$     & $0.5$ \\
    $tl$  & $ol$    & $ol$  & $0.5$     & $0.5$ \\
    $tl$  & $ol$    & $or$  & $0.5$     & $0.5$ \\
    $tl$  & $or$    & $li$  & $0.5$     & $0.5$ \\
    $tl$  & $or$    & $ol$  & $0.5$     & $0.5$ \\
    $tl$  & $or$    & $or$  & $0.5$     & $0.5$ \\
    $tr$  & $li$    & $li$  & $0$       & $1$   \\
    $tr$  & $li$    & $ol$  & $0.5$     & $0.5$ \\
    $tr$  & $li$    & $or$  & $0.5$     & $0.5$ \\
    $tr$  & $ol$    & $li$  & $0.5$     & $0.5$ \\
    $tr$  & $ol$    & $ol$  & $0.5$     & $0.5$ \\
    $tr$  & $ol$    & $or$  & $0.5$     & $0.5$ \\
    $tr$  & $or$    & $li$  & $0.5$     & $0.5$ \\
    $tr$  & $or$    & $ol$  & $0.5$     & $0.5$ \\
    $tr$  & $or$    & $or$  & $0.5$     & $0.5$ \\
    \bottomrule
\end{tabular}\hfill
\begin{tabular}{cccr}
    \toprule
    $S$   & $A_1$   & $A_2$ & $R$   \\
    \midrule
    $tl$  & $li$    & $li$  & $-2$  \\
    $tl$  & $li$    & $ol$  & $-101$\\
    $tl$  & $li$    & $or$  & $9$   \\
    $tl$  & $ol$    & $li$  & $-101$\\
    $tl$  & $ol$    & $ol$  & $-50$ \\
    $tl$  & $ol$    & $or$  & $-100$\\
    $tl$  & $or$    & $li$  & $9$   \\
    $tl$  & $or$    & $ol$  & $-100$\\
    $tl$  & $or$    & $or$  & $20$  \\
    $tr$  & $li$    & $li$  & $-2$  \\
    $tr$  & $li$    & $ol$  & $9$   \\
    $tr$  & $li$    & $or$  & $-101$\\
    $tr$  & $ol$    & $li$  & $9$   \\
    $tr$  & $ol$    & $ol$  & $20$  \\
    $tr$  & $ol$    & $or$  & $-100$\\
    $tr$  & $or$    & $li$  & $-101$\\
    $tr$  & $or$    & $ol$  & $-100$\\
    $tr$  & $or$    & $or$  & $-50$ \\
    \bottomrule
\end{tabular}\hfill 
\begin{tabular}{cccllll}
    \toprule
    $S$   & $A_1$   & $A_2$ & $\varOmega_{hl,hl}$ & $\varOmega_{hr,hl}$ & $\varOmega_{hl,hr}$ & $\varOmega_{hr,hr}$ \\
    \midrule
    $tl$  & $li$    & $li$  & $0.7225$  & $0.1275$  & $0.1275$  & $0.0225$  \\
    $tl$  & $li$    & $ol$  & $0.25$    & $0.25$    & $0.25$    & $0.25$  \\
    $tl$  & $li$    & $or$  & $0.25$    & $0.25$    & $0.25$    & $0.25$  \\
    $tl$  & $ol$    & $li$  & $0.25$    & $0.25$    & $0.25$    & $0.25$  \\
    $tl$  & $ol$    & $ol$  & $0.25$    & $0.25$    & $0.25$    & $0.25$  \\
    $tl$  & $ol$    & $or$  & $0.25$    & $0.25$    & $0.25$    & $0.25$  \\
    $tl$  & $or$    & $li$  & $0.25$    & $0.25$    & $0.25$    & $0.25$  \\
    $tl$  & $or$    & $ol$  & $0.25$    & $0.25$    & $0.25$    & $0.25$  \\
    $tl$  & $or$    & $or$  & $0.25$    & $0.25$    & $0.25$    & $0.25$  \\
    $tr$  & $li$    & $li$  & $0.7225$  & $0.1275$  & $0.1275$  & $0.0225$  \\
    $tr$  & $li$    & $ol$  & $0.25$    & $0.25$    & $0.25$    & $0.25$  \\
    $tr$  & $li$    & $or$  & $0.25$    & $0.25$    & $0.25$    & $0.25$  \\
    $tr$  & $ol$    & $li$  & $0.25$    & $0.25$    & $0.25$    & $0.25$  \\
    $tr$  & $ol$    & $ol$  & $0.25$    & $0.25$    & $0.25$    & $0.25$  \\
    $tr$  & $ol$    & $or$  & $0.25$    & $0.25$    & $0.25$    & $0.25$  \\
    $tr$  & $or$    & $li$  & $0.25$    & $0.25$    & $0.25$    & $0.25$  \\
    $tr$  & $or$    & $ol$  & $0.25$    & $0.25$    & $0.25$    & $0.25$  \\
    $tr$  & $or$    & $or$  & $0.25$    & $0.25$    & $0.25$    & $0.25$  \\
    \bottomrule
\end{tabular}
\end{center}
with $T_{s'} = T(s',s,a_1,a_2), s' \in \{tl,tr\}$ and $\varOmega_{o_1,o_2} = \varOmega((o_1,o_2),s,a_1,a_2)$, $o_1, o_2 \in \{hl,hr\}$.
The transition function only states that as long as both agents only listen, the state does not change (identity).
When at least one agent opens a door, the game basically restarts with the new state being set according to a uniform distribution.
It is basically a way of keeping the game infinite by resetting the state to an arbitrary one whenever the agents end the game by opening a door (to either lose---tiger, or win---gold).
One could argue that opening a door only ends the game and might not necessarily imply a restart.
In that case, one would keeping the state as is in all cases ($(1,0)$ distribution for all $tl$ lines, $(0,1)$ distribution for all $tr$ lines) and re-spawn the game with an arbitrary starting state, sampled from a $(0.5,0.5)$ distribution, for do-overs. 

The DecTiger model has the same action and observation sets for both agents and exhibits a counting symmetry.
Thus, it can be rewritten into a counting model (\cref{def:decpomdp:counold}) with $K=1$.
We index the one partition with $c$. 
\begin{itemize}[noitemsep] 
	\item $\bar{\boldsymbol{I}} = \mathfrak{I}_c = \{agent_1, agent_2\}$, 
	\item $S = \{tl,tr\}$, 
    \item $\bar{A}_c = \{\#_{\mathfrak{I}_c}[A_c]\}$, where the local action range is $\ran(A_c) = \ran(A_i) = \{li,ol,or\}$, $A_i$ from the ground DecPOMDP definition,
	\item $\bar{T}(S', S, \bar{A}_c) = P(S' \mid S, \bar{A}_c)$, 
	\item $\bar{R}(S, \bar{A}_c)$, 
    \item $\bar{O}_c = \{\#_{\mathfrak{I}_k}[O_k]\}$, where the local observation range is $\ran(O) = \{hl,hr\}$. 
	\item $\bar{\varOmega}(\bar{O}_c, S) = P( \bar{O}_c \mid S)$
\end{itemize}
with $\bar{T}$, $\bar{R}$, and $\bar{\varOmega}$ following after the specification of the policy-counted DecPOOMDP version.

In a policy-counted DecPOMDP (\cref{def:decpomdp:coun}), the action and observation variables are plain variables again.
The CRVs only occur within the transition, reward, and sensor functions. 
Specifically, for the policy-counted DecTiger example, the model is defined as follows:
\begin{itemize}[noitemsep] 
	\item $\bar{\boldsymbol{I}} = \mathfrak{I}_c = \{agent_1, agent_2\}$, 
	\item $S = \{tl,tr\}$, 
    \item $\bar{A} = \{A_c\}$, $\ran(A_c) = \ran(A_i) = \{li,ol,or\}$,
	\item $\bar{T}(S', S, \bar{A}_c) = P(S' \mid S, \bar{A}_c)$, 
	\item $\bar{R}(S, \bar{A}_c)$, 
    \item $\bar{O} = \{O_k\}_{k=1}^K$, $\ran(O_c) = \{hl,hr\}$. 
	\item $\bar{\varOmega}(\bar{O}_c, S) = P( \bar{O}_c \mid S)$
\end{itemize}
with $\bar{A}_c$ and $\bar{O}_c$ defined as for the counting DecPOMDP above.

The functions $\bar{T}$, $\bar{R}$, and $\bar{\varOmega}$ are given below in tabular notation with $\#A$ short for $\#_{\mathfrak{I}_k}[A_k]$ (histogram positions: $[li,ol,or]$) and $\bar{\varOmega}_h$ for $\varOmega(S,\#A,\#_{\mathfrak{I}_k}[O_k]=h)$:
\begin{center}
\begin{tabular}{ccll}
\toprule
	$S$	& $\#A$		& $\bar{T}_{tl}$& $\bar{T}_{tr}$\\
\midrule
	$tl$	& $[2,0,0]$	& $1.0$	& $0.0$	\\
	$tl$	& $[1,1,0]$	& $0.5$	& $0.5$	\\
	$tl$	& $[1,0,1]$	& $0.5$	& $0.5$	\\
	$tl$	& $[0,1,1]$	& $0.5$	& $0.5$	\\
	$tl$	& $[0,2,0]$	& $0.5$	& $0.5$	\\
	$tl$	& $[0,0,2]$	& $0.5$	& $0.5$	\\
	$tr$	& $[2,0,0]$	& $0.0$	& $1.0$	\\
	$tr$	& $[1,1,0]$	& $0.5$	& $0.5$	\\
	$tr$	& $[1,0,1]$	& $0.5$	& $0.5$	\\
	$tr$	& $[0,1,1]$	& $0.5$	& $0.5$	\\
	$tr$	& $[0,2,0]$	& $0.5$	& $0.5$	\\
	$tr$	& $[0,0,2]$	& $0.5$	& $0.5$	\\
    \bottomrule
\end{tabular} \hfill
\begin{tabular}{ccr}
\toprule
	$S$	& $\#A$		& $\bar{R}$	\\
\midrule
	$tl$	& $[2,0,0]$	& $-2$	\\
	$tl$	& $[1,1,0]$	& $-101$	\\
	$tl$	& $[0,2,0]$	& $-50$	\\
	$tl$	& $[0,1,1]$	& $-100$	\\
	$tl$	& $[0,0,2]$	& $20$	\\
	$tl$	& $[1,0,1]$	& $9$	\\
	$tr$	& $[2,0,0]$	& $-2$	\\
	$tr$	& $[1,1,0]$	& $9$	\\
	$tr$	& $[0,2,0]$	& $20$	\\
	$tr$	& $[0,1,1]$	& $-100$	\\
	$tr$	& $[0,0,2]$	& $-50$	\\
	$tr$	& $[1,0,1]$	& $-101$	\\
\bottomrule
\end{tabular}\hfill 
\begin{tabular}{cclll}
\toprule
	$S$	& $\#A$		& $\bar{\varOmega}_{[2,0]}$& $\bar{\varOmega}_{[1,1]}$& $\bar{\varOmega}_{[0,2]}$   \\
\midrule
	$tl$	& $[2,0,0]$	& $0.7225$	& $0.1275$	& $0.0225$	\\
	$tl$	& $[1,1,0]$	& $0.25$		& $0.25$		& $0.25$		\\
	$tl$	& $[0,2,0]$	& $0.25$		& $0.25$		& $0.25$		\\
	$tl$	& $[0,1,1]$	& $0.25$		& $0.25$		& $0.25$		\\
	$tl$	& $[0,0,2]$	& $0.25$		& $0.25$		& $0.25$		\\
	$tl$	& $[1,0,1]$	& $0.25$		& $0.25$		& $0.25$		\\
	$tr$	& $[2,0,0]$	& $0.7225$	& $0.1275$	& $0.0225$	\\
	$tr$	& $[1,1,0]$	& $0.25$		& $0.25$		& $0.25$		\\
	$tr$	& $[1,0,1]$	& $0.25$		& $0.25$		& $0.25$		\\
	$tr$	& $[0,2,0]$	& $0.25$		& $0.25$		& $0.25$		\\
	$tr$	& $[0,1,1]$	& $0.25$		& $0.25$		& $0.25$		\\
	$tr$	& $[0,0,2]$	& $0.25$		& $0.25$		& $0.25$		\\
\bottomrule
\end{tabular}
\end{center}
Note that the rows to not add up to one because the histogram $[1,1]$ in $\bar{\varOmega}$ stands for two joint observations, $(hl,hr)$ and $(hr,hl)$, meaning that the probability for $\bar{\varOmega}_{[1,1]}$ has to be counted twice for the probability distribution to add up to $1$.
In general, a multinomial coefficient provides the number of inputs represented by a histogram, i.e., ${n_k!}/{\prod_l n_l!}$.

Before moving on to model and policy space sizes, we list the DecTiger specification in the MADP toolbox.

\begin{listing}[H]
\caption{DecTiger specification in the MADP toolbox (without the comments from the source)}
\label{lst:dectiger}
\begin{lstlisting}
agents: 2
discount: 1
values: reward
states: tiger-left tiger-right
start: uniform
actions: 
listen open-left open-right
listen open-left open-right
observations:
hear-left hear-right
hear-left hear-right
# Transition probabilities
T: * : uniform
T: listen listen : identity
# Observation probabilities: <2 actions> : <state> : <2 observations> : probability
O: * : uniform
O: listen listen : tiger-left : hear-left hear-left : 0.7225
O: listen listen : tiger-left : hear-left hear-right : 0.1275
O: listen listen : tiger-left : hear-right hear-left : 0.1275
O: listen listen : tiger-left : hear-right hear-right : 0.0225
O: listen listen : tiger-right : hear-left hear-left : 0.7225
O: listen listen : tiger-right : hear-left hear-right : 0.1275
O: listen listen : tiger-right : hear-right hear-left : 0.1275
O: listen listen : tiger-right : hear-right hear-right : 0.0225
# Rewards: <2 actions> : <state> : * : * : reward
R: listen listen: * : * : * : -2
R: open-left open-left : tiger-left : * : * : -50
R: open-right open-right : tiger-right : * : * : -50
R: open-left open-left : tiger-right : * : * : 20
R: open-right open-right : tiger-left : * : * : 20
R: open-left open-right : tiger-left : * : * : -100
R: open-left open-right : tiger-right : * : * : -100
R: open-right open-left : tiger-left : * : * : -100
R: open-right open-left : tiger-right : * : * : -100
R: open-left listen : tiger-left : * : * : -101
R: listen open-right : tiger-right : * : * : -101
R: listen open-left : tiger-left : * : * : -101
R: open-right listen : tiger-right : * : * : -101
R: listen open-right : tiger-left : * : * : 9
R: listen open-left : tiger-right : * : * : 9
R: open-right listen : tiger-left : * : * : 9
R: open-left listen : tiger-right : * : * : 9
\end{lstlisting}
\end{listing}

\subsection{Model and Policy Space Sizes}
The following table characterises the ground, counting, and policy-counted model $M$, $\bar{M}_c$, $\bar{M}$ regarding different parameters:
\begin{itemize}[noitemsep]
    \item Number of agents: $N$, number of partitions: $K$, horizon $\tau$
    \item State size: $s$
    \item Size of a local / partition action: $a$, that is $|\ran(A_i)|$, $|\ran(\#_{\mathfrak{I}}[A_c])|$, $|\ran(A_c)|$, respectively
    \item Size of the joint (counted) action: $\boldsymbol{a}$
    \item Size of a local / partition observation: $o$, that is $|\ran(O_i)|$, $|\ran(\#_{\mathfrak{I}}[O_c])|$, $|\ran(O_c)|$, respectively
    \item Size of the joint (counted) action: $\boldsymbol{o}$
    \item Size of local / partition policy space per state: $p$, i.e., local policy $\pi$ over $A_i, O_i$, partition policy over $\#_{\mathfrak{I}}[A_c]), \#_{\mathfrak{I}}[O_c])$, counted policy over representative policies over $A_k, O_k$
    \item Size of the joint (counted) policy space per state: $\boldsymbol{p}$
    \item Size of the transition, reward, and sensor functions by range sizes of the inputs: $T$, $R$, $\varOmega$
\end{itemize}
\begin{table}[H]
\centering
\begin{tabular}[t]{c c c c}
    \toprule
                        & $M$ & $\bar{M}_c$                                 & $\bar{M}$\\
     \midrule
     $s$                & $2$ & $2$                                         & $2$\\[1em]
     $a$                & $3$ & $\binom{N + 3 -1}{3-1} = \binom{N+2}{2}$    & $3$\\[1em]
     $\boldsymbol{a}$   & $3^N$ & $\binom{N + 2}{2}^K$                      & $\binom{N + 2}{2}^K$\\[1em]
     $o$                & $2$ & $\binom{N + 2 -1}{2-1} = \binom{N+1}{1}$    & $2$\\[1em]
     $\boldsymbol{o}$   & $2^N$ & $\binom{N + 1}{1}^K$                      & $\binom{N + 1}{1}^K$\\[1em]
     $p$                & $3^{\frac{2^\tau - 1}{2-1}}$ & $\binom{N + 2}{2}^{\frac{\binom{N + 1}{1}^\tau -1}{\binom{N + 1}{1} -1}}$                              & $\begin{pmatrix}{N + 3^{\frac{2^\tau}{2-1}} - 1}\\{3^{\frac{2^\tau}{2-1}} -1}\end{pmatrix}$\\[1em]
     $\boldsymbol{p}$   & $3^{N\frac{2^\tau - 1}{2-1}}$ & $\binom{N + 2}{2}^{K\frac{\binom{N + 1}{1}^\tau -1}{\binom{N + 1}{1} -1}}$                             & $\begin{pmatrix}{N + 3^{\frac{2^\tau}{2-1}} - 1}\\{3^{\frac{2^\tau}{2-1}} -1}\end{pmatrix}^K$\\[1.5em]
     $T$            & $2\cdot 2 \cdot 3^N$ & $2\cdot 2 \cdot \binom{N + 2}{2}^K$ & $2\cdot 2 \cdot \binom{N + 2}{2}^K$\\[1em]
     $R$            & $2 \cdot 3^N$     & $2 \cdot \binom{N + 2}{2}^K$ & $2 \cdot \binom{N + 2}{2}^K$\\[1em]
     $\varOmega$    & $2 \cdot 2^N$     & $2 \cdot \binom{N+1}{1}^K$    & $2 \cdot \binom{N+1}{1}^K$\\
     \bottomrule
\end{tabular}
\end{table}

\subsection{Dynamic Programming}
We consider the three steps of dynamic programming (backup, value calculation, pruning) for the levels $t=0$ and $t=1$ for the three versions of the DecTiger example.

\subsubsection{Level $t=0$}
\paragraph{Backup: }
The operator for the ground version builds \emph{three} policies for each of its \underline{two} agents, each policy having a single (root) node for one of the actions $li,ol,or$, building six policies overall.
A counting programming operator for the counting version would build \emph{six} policies for its \underline{one} partition, each with a single (root) node for one of the action histograms $[2,0,0],\allowbreak [1,1,0],\allowbreak [1,0,1],\allowbreak [0,1,1],\allowbreak [0,2,0],\allowbreak [0,0,2]$.
The counted programming operator for the policy-counted version builds \emph{three} representative policies for its \emph{one} partition, each with a single (root) node for one of the actions $li,ol,or$.
With two available agents, there are six counted policies over the three representative ones: $[2,0,0],\allowbreak [1,1,0],\allowbreak [1,0,1],\allowbreak [0,1,1],\allowbreak [0,2,0],\allowbreak [0,0,2]$.

\paragraph{Value Computation: }
For the ground version, the operator calculates the values for each of its \emph{two} agents for each of its \underline{three} policies.
For each policy, the computation has to be performed for each state $tl,tr$ and each of the other agent's three policies.
For the counting version, the operator calculates the values for its \emph{one} partition for each of its \emph{six} policies.
For each policy, the computation has to be performed for each state $tl,tr$.
For the policy-counted version, the operator calculates the values for its \emph{one} partition for each of its \emph{six} counted policies.
For each policy, the computation has to be performed for each state $tl,tr$.

\paragraph{Pruning: }
In the ground version, the operator checks for each of the \emph{two} agents, whether one of its \underline{three} policies is (weakly) dominated by one of the other two policies for each state $tl,tr$ and each of the three policies of the other agent (does not need to be the same policy over all states and policies of the other agent), until there are no more policies to prune.
In the counting version, the operator would check for its \emph{one} partition, whether any of its \underline{six} policies is (weakly) dominated by one of the other five policies for each state $tl,tr$, until there are no more policies to prune.
In the policy-counted version, the operator checks for its \emph{one} partition, whether any of its \underline{three} representative policies is dominated by one of the other two policies for each state $tl,tr$ and each of the other counted policies over the remaining partition, which contains only one agent, i.e., $[1,0,0],\allowbreak [0,1,0],\allowbreak [0,0,1]$ (that is just the three representative policies), until there are no more policies to prune.

\paragraph{Comment: } 
The difference between the counting and policy-counted versions is not as pronounced at this level, with the ground version also working within the same order of magnitude.

\subsubsection{Level $t=1$}
\paragraph{Backup: }
With no policies pruned, the operator for the ground version builds for its \emph{two} agent each \underline{$27$} policies out of its previous three policies and two possible observations $hl,hr$, with each policy having a root node and two children, one for each observation.
The actions in the nodes are each possible combination of $li,ol,or$.
With no policies pruned, the operator for the counting version would build for its \emph{one} partition \emph{$216$} policies out of its previous six policies and three possible counted observations $[0,2],[1,1],[2,0]$.
The policies would have a root node and three children, one for each possible histogram, with the nodes containing each possible combination of the six action histograms.
The operator for the policy-counted version builds the same \underline{$27$} (representative) policies out of its previous three representative policies and two possible observations $hl,hr$ that the ground version has.
With two available agents, there are then $378$ counted policies.

\paragraph{Value Computation: }
For the ground version, the operator performs the same calculations outlined for level $t=0$ for its $27$ policies for each combination of state $tl,tr$ and policy of the other agent ($27$ as well) for each agent.
For the counting version, the operator calculates the values for its one partition for each of its $216$ policies for each state $tl,tr$.
For the policy-counted version, the operator calculates the values for its one partition for each of the $378$ counted policies.
For each policy, the computation has to be performed for each state $tl,tr$.

\paragraph{Pruning: }
Pruning again follows the same procedure as outlined above, for {two} agents and each of their {$27$} policies in each state and for each of the other agent's policies in the ground version.
In the counting version, the {$1296$} policies of its one partition get checked for each state $tl,tr$.
In the policy-counted version, the {$27$} representative policies of its one partition are checked for each state and each counted policy of the remaining partition, i.e., the $27$ representative policies.

\paragraph{Comment: } 
This numerical example illustrates impressively that policy counting only pays off if the number of agents is large or at least outweighs the number of representative policies.
That is the number of agents should at least be an order larger than the number of representative actions and observations available.

\end{document}